\documentclass[final, nonatbib]{article}

\usepackage[final]{neurips_2023}

\usepackage[utf8]{inputenc} 
\usepackage[T1]{fontenc}    
\usepackage{hyperref}       
\hypersetup{colorlinks}
\usepackage{url}            
\usepackage{booktabs}       
\usepackage{amsfonts}       
\usepackage{nicefrac}       
\usepackage{microtype}      
\usepackage{xcolor}         
\usepackage{wrapfig}
\usepackage{algorithm}
\usepackage[noend]{algpseudocode}

\usepackage{amsmath}
\usepackage{amssymb}
\usepackage{mathtools}
\usepackage{amsthm}
\usepackage{enumerate}
\usepackage{enumitem}

\theoremstyle{plain}
\newtheorem{theorem}{Theorem}[section]

\newtheorem{lemma}[theorem]{Lemma}

\theoremstyle{definition}
\newtheorem{definition}[theorem]{Definition}
\newtheorem{assumption}[theorem]{Assumption}
\newtheorem{remark}[theorem]{Remark}

\def\liminf{\mathop{\rm lim\,inf}\limits}

\def\bZ{\mathbf{Z}}
\def\Q{\mathbf{Q}}
\def\R{\mathbb{R}}
\def\C{\mathbb{C}}
\def\E{\mathbb{E}}
\def\P{\mathbb{P}}

\def\L{\mathcal{L}}

\def\H{\mathbf{H}}

\def\eps{\varepsilon}

\def\T{\mathbf{T}}
\def\S{\mathbf{S}}
\def\D{\mathbf{D}}
\def\X{\mathbf{X}}
\def\x{\mathbf{x}}
\def\v{\mathbf{v}}
\def\u{\mathbf{u}}
\def\W{\mathbf{W}}
\def\a{\mathbf{a}}
\def\A{\mathbf{A}}
\def\B{\mathbf{B}}
\def\C{\mathbf{C}}

\def\var{\textup{Var}}

\def\param{\boldsymbol{\theta}}
\def\Param{\boldsymbol{\Theta}}
\def\I{\mathbf{I}}
\def\h{\mathbf{h}}
\def\bSigma{\boldsymbol{\Sigma}}

\def\V{\mathbf{V}}
\def\U{\mathbf{U}}

\def\bR{\mathbf{R}}
\def\w{\mathbf{w}}
\def\blambda{\boldsymbol{\lambda}}

\def\beps{\boldsymbol{\eps}}
\def\rank{\textup{rank}}
\def\bSigma{\boldsymbol{\Sigma}}
\def\bphi{\boldsymbol{\phi}}
\def\bpsi{\boldsymbol{\psi}}
\def\bPhi{\boldsymbol{\Phi}}

\def\L{\mathcal{L}}

\DeclareMathOperator{\diag}{diag}

\DeclareMathOperator*{\argmin}{arg\,min}
\DeclareMathOperator{\vect}{vec}

\newcommand{\tr}{\textup{tr}}
\newcommand{\p}{\mathbf{p}}

\def\liminf{\mathop{\rm lim\,inf}\limits}

\def\R{\mathbb{R}}
\def\E{\mathbb{E}}

\def\P{\mathbb{P}}

\def\L{\mathcal{L}}

\def\H{\mathbf{H}}

\def\eps{\varepsilon}

\def\X{\mathbf{X}}
\def\x{\mathbf{x}}
\def\g{\mathbf{g}}
\def\W{\mathbf{W}}
\def\Y{\mathbf{Y}}
\def\bP{\mathbf{P}}
\def\H{\mathbf{H}}
\def\Beta{\boldsymbol{\beta}}
\def\bGamma{\boldsymbol{\Gamma}}
\def\bgamma{\boldsymbol{\gamma}}
\def\A{\mathbf{A}}
\def\B{\mathbf{B}}
\def\var{\textup{Var}}

\def\y{\mathbf{y}}
\def\u{\mathbf{u}}

\title{Exponentially Convergent Algorithms for \\ Supervised Matrix Factorization}

\author{%
	Joowon Lee \\ 
	Department of Statistics \\
	University of Wisconsin - Madison, WI, USA \\
	\texttt{jlee2256@wisc.edu} \\
	\And
	Hanbaek Lyu \\
	Department of Mathematics \\
	University of Wisconsin - Madison, WI, USA \\
	\texttt{hlyu@math.wisc.edu} \\
	\AND
	Weixin Yao \\
	Department of Statistics \\
	University of California, Riverside, CA, USA \\
	\texttt{weixiny@ucr.edu} \\
}

\pdfoutput=1

\begin{document}

	\maketitle

	\begin{abstract}
		Supervised matrix factorization (SMF) is a classical machine learning method that simultaneously seeks feature extraction and classification tasks, which are not necessarily a priori aligned objectives.
		Our goal is to use SMF to learn low-rank latent factors that offer interpretable, data-reconstructive, and class-discriminative features, addressing challenges posed by high-dimensional data. Training SMF model involves solving a nonconvex and possibly constrained optimization with at least three blocks of parameters. Known algorithms are either heuristic or provide weak convergence guarantees for special cases. In this paper, we provide a novel framework that `lifts' SMF as a low-rank matrix estimation problem in a combined factor space and propose an efficient algorithm that provably converges exponentially fast to a global minimizer of the objective with arbitrary initialization under mild assumptions. Our framework applies to a wide range of SMF-type problems for multi-class classification with auxiliary features. To showcase an application, we demonstrate that our algorithm successfully identified well-known cancer-associated gene groups for various cancers.
\end{abstract}

\section{Introduction}

In classical classification models, such as logistic regression, a conditional class-generating probability distribution is modeled as a simple function of the observed features with unknown parameters to be trained. However, the raw observed features may be high-dimensional, and most of them might be uninformative and hard to interpret (e.g., pixel values of an image). Therefore, it would be desirable to extract more informative and interpretable low-dimensional features prior to the classification task. For instance, the multi-layer perceptron or deep neural networks (DNN) in general \cite{bishop1995neural, bishop2006pattern} use additional feature extraction layers prior to the logistic regression layer. This allows the model itself to learn the most effective (supervised) feature extraction mechanism and the association of the extracted features with class labels simultaneously.

Matrix factorization (MF) is a classical unsupervised feature extraction framework, which learns latent structures of complex datasets and is regularly applied in the analysis of text and images \cite{elad2006image, mairal2007sparse, peyre2009sparse}. Various matrix factorization models such as singular value decomposition (SVD), principal component analysis (PCA), and nonnegative matrix factorization (NMF) provide fundamental tools for unsupervised feature extraction tasks \cite{golub1971singular, wall2003singular, abdi2010principal, lee1999learning}. Extensive research has been conducted to adapt matrix factorization models to perform classification tasks by supervising the matrix factorization process using additional class labels. Note that matrix factorization and classification are not necessarily aligned objectives, so some degree of trade-off is necessary when seeking to achieve both goals simultaneously. \textit{Supervised matrix factorization} (SMF) provides systematic approaches for such multi-objective tasks. Our goal is to use SMF to learn low-rank latent factors that offer interpretable, data-reconstructive, and class-discriminative features, addressing challenges posed by high-dimensional data. The general framework of SMF was introduced in \cite{mairal2008supervised}. A similar SMF-type framework of discriminative K-SVD was proposed for face recognition \cite{zhang2010discriminative}. A stochastic formulation of SMF was proposed in \cite{mairal2011task}. SMF has also found numerous applications in various other problem domains, including speech and emotion recognition \cite{gangeh2014multiview}, music genre classification \cite{zhao2015supervised}, concurrent brain network inference \cite{zhao2015supervised}, structure-aware clustering \cite{yankelevsky2017structure}, and object recognition \cite{li2019discriminative}. Recently, supervised variants of NMF, as well as PCA, were proposed in \cite{austin2018fully, leuschner2019supervised, ritchie2020supervised}. See also the survey work of \cite{gangeh2015supervised} on SMF.

Various SMF-type models have been proposed in the past two decades. We divide them into two categories depending on whether the extracted low-dimensional feature or the feature extraction mechanism itself is supervised. We refer to them as feature-based and filter-based SMF, respectively. Feature-based SMF models include the classical ones by Mairal et al. (see, e.g., \cite{mairal2008supervised, mairal2011task}) as well as the more recent model of convolutional matrix factorization by \cite{kim2016convolutional} for a contextual text recommendation system. Filter-based SMF models have been studied more recently in the supervised matrix factorization literature, most notably from supervised nonnegative matrix factorization \cite{austin2018fully, leuschner2019supervised} and supervised PCA \cite{ritchie2020supervised}. 

\paragraph{Contributions}

In spite of vast literature on SMF, due to the high non-convexity of the associated optimization problem (see \eqref{eq:SMF_main}), 
algorithms for SMF mostly lack rigorous convergence analysis and there has not been any algorithm that provably converges to a global minimizer of the objective at an exponential rate. We summarize our contributions below.

\begin{description}[itemsep=0cm, leftmargin=0.5cm]
	\item{$\bullet$} We formulate a general class of SMF-type  models (including both the feature- and the filter-based ones) with high-dimensional features as well as low-dimensional auxiliary features (see \eqref{eq:SMF_main}).  
	
	\item{$\bullet$} We provide a novel framework that `lifts' SMF as a low-rank matrix estimation problem in a combined factor space and propose an efficient algorithm that converges exponentially fast to a global minimizer of the objective with an arbitrary initialization (Theorem \ref{thm:SMF_LPGD})
	. We numerically validate our theoretical results (see Fig. \ref{fig:benchmark_MNIST}).

	\item{$\bullet$} We theoretically compare the robustness of filter-based and feature-based SMF, establishing that the former is computationally more robust (see Theorem \ref{thm:SMF_LPGD}) 
	while the latter is statistically more robust (see Theorem \ref{thm:SMF_LPGD_STAT}).
	
	\item{$\bullet$} Applying our method to microarray datasets for cancer classification, we show that not only it is competitive against benchmark methods, but it is able to identify groups of genes including well-known cancer-associated genes (see Fig. \ref{fig:pancreatic_cancer}).
\end{description}

\subsection{Notations}
\label{subsection:notation}

Throughout this paper, we denote by $\R^{p}$ the ambient space for data equipped with standard inner project $\langle\cdot, \cdot \rangle$ that induces the Euclidean norm $\lVert \cdot \rVert$. We denote by  $\{ 0,1,\dots,\kappa \}$ the space of class labels with $\kappa+1$ classes.   For a convex subset $\Param$ in an Euclidean space, we denote $\Pi_{\Param}$ the projection operator onto $\Param$. For an integer $r\ge 1$, we denote by $\Pi_{r}$ the rank-$r$ projection operator for matrices. 
For a matrix $\A=(a_{ij})_{ij}\in \R^{m\times n}$, we denote its Frobenius, operator (2-), and supremum norm by $\lVert \A \rVert_{F}^{2} := \sum_{i,j} a_{ij}^{2},  \lVert \A \rVert_{2} := \sup_{\x\in \R^{n},\, \lVert \x \rVert=1} \, \lVert \A\x \rVert,  \lVert \A \rVert_{\infty}:= \max_{i,j} |a_{ij}|$, 
respectively. For each $1\le i \le m$ and $1\le j \le n$, we denote $\A[i,:]$ and $\A[:,j]$ for the $i$th row and the $j$th column of $\A$, respectively. For each integer $n\ge 1$, $\I_{n}$ denotes the $n\times n$ identity matrix. For square symmetric matrices  $\A,\B\in \R^{n\times n}$, we denote $\A\preceq  \B$ if $\v^{T}\A\v \le \v^{T}\B\v $ for all unit vectors $\v\in \R^{n}$.	
For two matrices $\A$ and $\B$, we denote $[\A, \B]$ and $[\A \parallel \B]$ the matrices obtained by concatenating (stacking) them by horizontally and vertically, respectively, assuming matching dimensions.

\subsection{Model setup}\label{subsection:SMF1}

Suppose we are given with $n$ labeled signals $(y_{i}, \x_{i},\x_{i}')$ for $i=1,\dots, n$, where $y_{i} \in \{ 0,1,\dots,\kappa \}$ is the label, $\x_{i}\in \R^{p}$ is a high-dimensional feature of $i$, and $\x_{i}'\in \R^{q}$ is a low-dimensional auxiliary feature of $i$ ($p\gg q$). For a vivid context, think of $\x_{i}$ as the X-ray image of a patient $i$ and $\x_{i}'$ denoting some biological measurements, such as gender, smoking status, and body mass index. When making predictions of $y_{i}$, we use a suitable $r\, (\ll p)$ dimensional compression of the high-dimensional feature $\x_{i}$ as well as the low-dimensional feature $\x_{i}'$ as-is. We assume such compression is done by some  matrix of \textit{(latent) factors} 
$\W=[\w_{1},\dots,\w_{r}]\in \R^{p\times r}$ that is \textit{reconstructive} in the sense that the observed signals $\x_{i}$ can be 
reconstructed as (or approximated by) the linear transform of the `atoms' $\w_{1},\dots,\w_{r}\in \R^{p}$ for some suitable `code' $\h_{i}\in \R^{r}$. More concisely, $\X_{\textup{data}}=[\x_{1},\dots,\x_{n}]\approx \W \H$, where $\H=[\h_{1},\dots, \h_{n}]\in \R^{r\times n}$. In practice, we can choose $r$ to be the approximate rank of data matrix $\X_{\textup{data}}$ (e.g., by finding the elbow of the scree plot).

\begin{figure}[h!]
	\centering
	\includegraphics[width=1\linewidth]{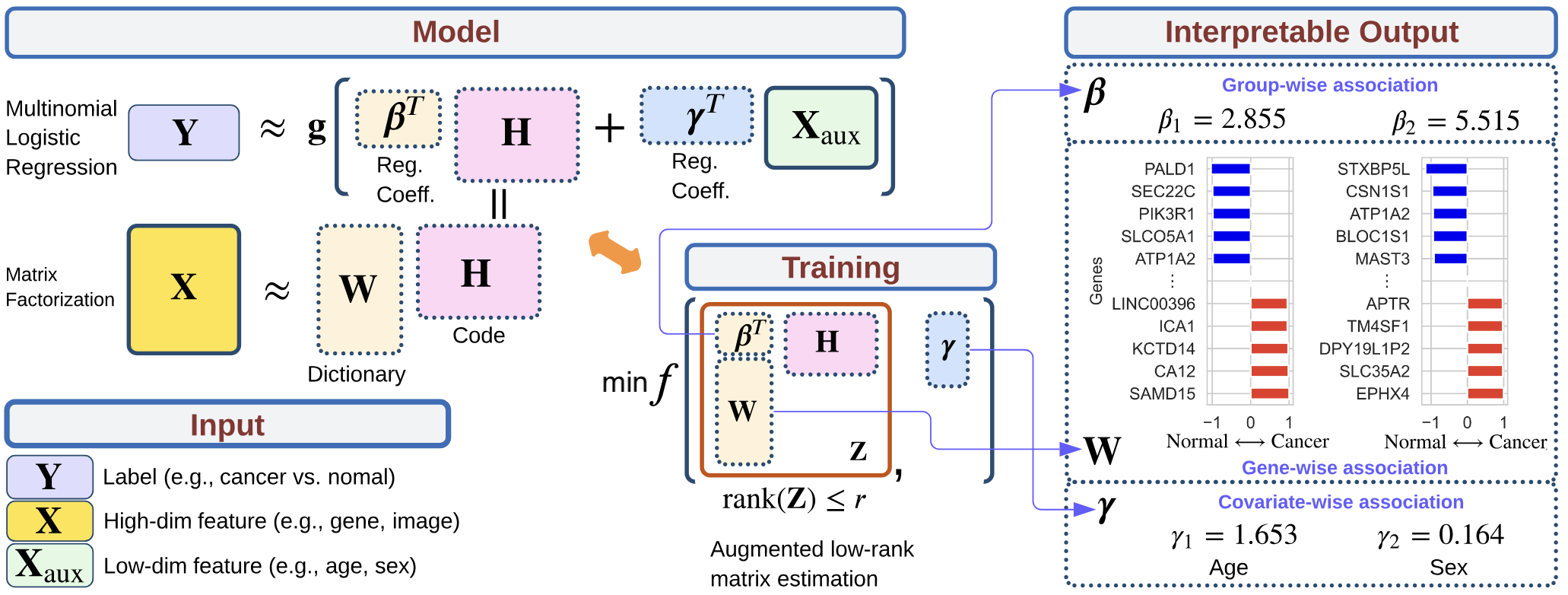} 
	\vspace{-0.5cm}
	\caption{Overall scheme of the proposed method for SMF-$\H$.}
	\label{fig:flowchart}
\end{figure}

Now, we state our probabilistic modeling assumption. Fix parameters $\W\in \R^{p\times r}$, $\h_{i}\in \R^{r}$, $\Beta\in \R^{r\times \kappa}$, and $\bgamma\in \R^{q\times \kappa}$. Let $h:\R\rightarrow [0,\infty)$ be a \textit{score function} (e.g., $h(\cdot)=\exp(\cdot)$ for multinomial logistic regression). We assume $y_{i}$ is a realization of a random variable whose conditional distribution is specified as 
\begin{align}\label{eq:SMF_model}
	\left[ \P\left( y_{i}=0\,|\, \x_{i}, \x_{i}'\right),\dots,  \P\left( y_{i}=\kappa\,|\, \x_{i}, \x_{i}'\right)\right] = \mathbf{g}(\a_{i}) := C [1, h(\a_{i,1}),\dots, h(\a_{i,\kappa})], 
\end{align}
where $C$ is the normalization constant and $\a_{i}=(\a_{i,1},\dots,\a_{i,\kappa})\in \R^{\kappa}$ is the \textit{activation} for $y_{i}$ defined in two ways, depending on whether we use a `feature-based' or `filter-based' SMF model:
\begin{align}\label{eq:activation}
	&\a_{i}  = \begin{cases}
		\Beta^{T} \h_{i} + \bgamma^{T} \x'_{i} & \textup{for feature-based (SMF-$\H$)},\\
		\Beta^{T} \W^{T} \x_{i} + \bgamma^{T} \x'_{i} & \textup{for filter-based (SMF-$\W$)}. 
	\end{cases}
\end{align}
One may regard $(\Beta, \bgamma)$ as the `multinomial regression coefficients' with input feature $(\h_{i}, \x_{i}')$ or $(\W^{T}\x_{i},\x_{i}')$. In \eqref{eq:activation}, we may regard the code $\h_{i}$ (coming from $\x_{i}\approx \W \h_{i}$) or the `filtered signal' $\W^{T}\x_{i}$ as the $r$-dimensional compression of $\x_{i}$. Note that these two coincide if we have perfect factorization $\x_{i}=\W \h_{i}$ and the factor matrix $\W$ are orthonormal, i.e., $\W^{T}\W=\I_{r}$, but we do not necessarily make such an assumption. 

There are some notable differences between SMF-$\H$ and SMF-$\W$ when predicting the unknown label of a test point.     If we are given a test point $(\x_{\textup{test}}, \x_{\textup{test}}')$, the predictive probabilities for its unknown label $y_{\textup{test}}$ is given by \eqref{eq:SMF_model} with activation $\a$ computed as in \eqref{eq:activation}. 
This only involves straightforward matrix multiplications for SMF-$\W$, which can also be viewed as a forward propagation in a multilayer perceptron \cite{murtagh1991multilayer} with $\W$ acting as the first layer weight matrix (hence named `filter'). However, for SMF-$\H$, one needs to solve additional optimization problems for testing. Namely, for every single test signal $(\x_{\textup{test}},\x_{\textup{test}}')$, its correct code representation $\h_{\textup{test}}$ needs to be learned by solving the following `supervised sparse coding' problem (see \cite{mairal2008supervised}): 
\begin{align}\label{eq:supervised_sparse_coding}
	\min_{y\in \{0,1,\dots,\kappa\} }  \min_{\h}\, \ell(y, \Beta^{T}\h + \bgamma^{T}\x_{\textup{test}}' ) + \xi \lVert \x_{\textup{test}} - \W \h \rVert_{F}^{2}. 
\end{align}
A more efficient heuristic testing method for SMF-$\H$ is by approximately computing $\h_{\textup{test}}$ by only minimizing the second term in \eqref{eq:supervised_sparse_coding}.

In order to estimate the model parameters $(\W, \H, \Beta,\bgamma)$ from observed training data $(\x_{i},y_{i})$  for $i=1,\dots, n$, we consider the following multi-objective optimization problem: 
\begin{align}\label{eq:SMF_main}
	\min_{\W,\H,\Beta,\bgamma } \,\, &   \sum_{i=1}^{n}  \ell(y_{i},\a_{i})  +   \xi \lVert \X_{\textup{data}} - \W\H\rVert_{F}^{2},
\end{align}
where $\X_{\textup{data}}=[\x_{1},\dots,\x_{n}]\in \R^{p\times n}$, $\a_{i}$ is as in \eqref{eq:activation}, and $\ell(\cdot)$ is the classification loss measured by the negative log-likelihood: 
\begin{align}\label{eq:loglikelihood}
	\ell(y, \a):=\log \left( 1+\sum_{c=1}^{\kappa} h(a_{c}) \right)   - \sum_{c=1}^{\kappa}\mathbf{1}_{\{y=c\} }\log h(a_{c}).
\end{align}
In \eqref{eq:SMF_main}, the \textit{tuning parameter} $\xi$ controls the trade-off between the two objectives of classification and matrix factorization. The above is a nonconvex problem involving four blocks of parameters that could have additional constraints (e.g., bounded norm).  
This problem entails many classical models as special cases. When 
$\xi\gg 1$ so that effectively only the second term in \eqref{eq:SMF_main} is being minimized with respect to $\W$ and $\H$, it becomes the classical matrix factorization problem \cite{mairal2010online, mairal2013optimization,mairal2013stochastic}, where one seeks to find factor matrix $\W$ that can best reconstruct the feature vectors $\X_{\textup{data}}$ via the factorization $\X_{\textup{data}}\approx \W \H$.
In Figure \ref{fig:pancreatic_cancer}, we will demonstrate that the best reconstructive factor matrix could be significantly different from the supervised factor matrix learned by SMF and may not be very effective for classification tasks.

\subsection{Related works}
The SMF training problem \eqref{eq:SMF_main} is a nonconvex and possibly constrained optimization problem, generally with non-unique minimizers.  Since it is difficult to solve exactly, approximate procedures such as block coordinate descent (BCD)  (see, e.g., \cite{wright2015coordinate}) are often used. Such methods utilize the fact that the objective function in \eqref{eq:SMF_main} is convex in each of the four (matrix) variables. Such an algorithm proceeds by iteratively optimizing for only one block while fixing the others (see \cite{mairal2008supervised, austin2018fully, leuschner2019supervised, ritchie2020supervised}). However, convergence analysis or statistical estimation bounds of such algorithms are quite limited. Appealing to general convergence results for BCD methods (e.g., \cite{grippo2000convergence, xu2013block}), one can at most guarantee asymptotic convergence to the  stationary points or polynomial 
convergence to Nash equilibria or of the objective \eqref{eq:SMF_main}, modulo carefully verifying the assumptions of these general results.
We also remark that \cite{mairal2011task} provided a rigorous justification of the differentiability of a feature-based SMF model. 

The main finding of our work is 
that the non-convexity of the SMF problem \eqref{eq:SMF_main} is `benign', in the sense that \textit{there exists an algorithm globally convergent to a global optimum at an exponential rate.} We use a `double-lifting' technique that converts the nonconvex SMF problem \eqref{eq:SMF_main} into a low-rank factored estimation with a convex objective. This is reminiscent of the tight relation between a low-rank matrix estimation and a nonconvex factored estimation problem, which has been actively employed in a body of works in statistics and optimization \cite{agarwal2010fast, ravikumar2011high, negahban2011estimation, zheng2015convergent, tu2016low, wang2017unified,park2017non, park2018finding, tong2021accelerating}. Our exponentially convergent SMF algorithms are versions of low-rank projected gradient descent in the algorithm \eqref{eq:LRPGD_iterate0} that operate in the double-lifted space.

\section{Methods}

\subsection{Sketch of key idea}
\label{subsection:sketch}

Our key idea to solve \eqref{eq:SMF_main} is to transform it into a variant of the low-rank matrix estimation problem \eqref{eq:SMF_regression_H3} and then use a \textit{Low-rank Projected Gradient Descent} (LPGD) algorithm~\eqref{eq:LRPGD_iterate_main}:
\begin{minipage}[t]{.35\textwidth}
	\vspace{-0.2cm}
	\begin{align}\label{eq:SMF_regression_H3}
		\min_{\bZ=[\param,\bgamma]\in \Param,\, \rank(\param)\le r}
		f\left( \bZ \right) \quad 
	\end{align}
\end{minipage}
\begin{minipage}[t]{.65\textwidth}
	\vspace{-0.35cm}
	\begin{align}\label{eq:LRPGD_iterate_main}
		\qquad	\bZ_{t} \leftarrow \Pi_{r} \bigg(\Pi_{\Param} \left(\bZ_{t-1} - \tau \nabla f(\bZ_{t-1}) \right) \bigg), \, \textup{$\tau>0$ fixed}.
	\end{align}
\end{minipage}

\noindent In \eqref{eq:SMF_regression_H3}, one seeks to minimize an objective $f$ w.r.t. a paired matrix parameter $\bZ=[\param,\bgamma]$ within a convex constraint set $\Param$ and an additional rank constraint $\rank(\param)\le r$. In \eqref{eq:LRPGD_iterate_main}, $\Pi_{r}$ denotes applying rank-$r$ projection on the first factor $\param$ while keeping $\bgamma$ the same. Problem \eqref{eq:SMF_regression_H3} encompasses a variety of significant problems, such as matrix regression \cite{candes2012exact, negahban2009unified} and matrix completion \cite{rohde2011estimation, negahban2012restricted}. 
For these problems, algorithms of type \eqref{eq:LRPGD_iterate_main} have been studied in \cite{park2018finding,wang2017unified}. 

To illustrate how the SMF problem \eqref{eq:SMF_main} transforms into a low-rank matrix estimation \eqref{eq:SMF_regression_H3}, we consider a much simpler version of SMF-$\H$. That is, instead of combining matrix factorization with multinomial logistic regression for multi-class classification problems, we combine it with \textbf{linear regression}. Thus, the response variable $y$ in this discussion assumes all values in the real line. For additional simplicity, we assume there are no auxiliary features $\X_{\textup{aux}}$.
We seek to solve matrix factorization and linear regression problems simultaneously for data matrix $\X_{\textup{data}}\in \R^{p\times n}$ and response variable $\Y\in \R^{1\times n}$: 
\begin{align}
	\min_{\W, \H, \Beta } \lVert \Y - \Beta^{T}\H  \rVert_{F}^{2} + \xi  \lVert \X_{\textup{data}} - \W \H  \rVert_{F}^{2}   .
\end{align}
This is a three-block optimization problem involving three factors $\W\in \R^{p\times r},\H\in \R^{r\times n}$ and $\Beta\in \R^{r\times 1}$, which is nonconvex and computationally challenging to solve exactly. Instead, consider reformulating this nonconvex problem as the following matrix factorization problem:
\begin{align}\label{eq:SMF_regression_H1}
	\min_{\W, \H, \Beta } f\left( \begin{bmatrix} \Beta^{T} \\  \W  \end{bmatrix} \H \right)
	:=
	\left\lVert 
	\begin{bmatrix}
		\Y \\
		\sqrt{\xi} \X_{\textup{data}}
	\end{bmatrix}
	- 
	\begin{bmatrix}
		\Beta^{T}\\
		\sqrt{\xi} 	\W 
	\end{bmatrix}
	\H \right\rVert_{F}^{2}.
\end{align}
Indeed, we now seek to find \textit{two} decoupled matrices (instead of three), one for $\Beta^{T}$ and $ \W$ stacked vertically, and the other for $\H$. A similar idea of matrix stacking was used in \cite{zhang2010discriminative} for discriminative K-SVD. Proceeding one step further, another important observation we make is that it is also equivalent to finding a \textit{single} matrix $\param:=\begin{bmatrix} \Beta^{T}\H \parallel  \W\H \end{bmatrix}\in \R^{(1+p)\times n}$ of rank at most $r$ that minimizes the function $f$ in  \eqref{eq:SMF_regression_H1}, which is convex (specifically, quadratic) in $\param$: (See Fig. \ref{fig:flowchart} Training).

For SMF-W, consider the following analogous  linear regression model:
\begin{align}\label{eq:SMF_regression_W1}
	\min_{\W, \H, \Beta } f\left( \W [\Beta,  \H  ] \right):= \begin{matrix} 
		\lVert \Y -  \Beta^{T}  \W^{T}\X_{\textup{data}} \rVert_{F}^{2} 
		+ \xi  \lVert \X_{\textup{data}} - \W \H  \rVert_{F}^{2},
	\end{matrix}
\end{align}
where the right-hand side above is obtained by replacing $\H$ with $\W^{T}\X_{\textup{data}}$ in \eqref{eq:SMF_regression_H1}. 
Note that the objective function depends only on the product of the two matrices $\W$ and $[\Beta,\H]$. Then, we may further lift it as the low-rank matrix estimation problem by seeking a single matrix $\param:=[\W\Beta,\, \W\H]\in \R^{p\times (1+n)}$ of rank at most $r$ that solves \eqref{eq:SMF_regression_H3} with $f$ being the function in \eqref{eq:SMF_regression_W1}.

\subsection{Algorithm}
\label{sec:Algorithm}

Motivated by the observation we made before, we rewrite SMF-$\H$ in \eqref{eq:SMF_main} as 
\begin{align}\label{eq:SMF_feat_CALE0}
	\min_{\substack{[\param, \bgamma] \in \Param \\  \rank(\param)\le r}}	 F	\left(\param,\, \bgamma \right)
	&:=    \sum_{i=1}^{n} \ell(y_{i}, \A[:,i]  + \bgamma^{T}\x_{i}' )   +   \xi \lVert  \X_{\textup{data}}  -\B\rVert_{F}^{2} +  \lambda \left( \lVert \A \rVert_{F}^{2} + \lVert \bgamma \rVert_{F}^{2}  \right),   
\end{align}
where $\A = \Beta^{T}\H$, $\B = \W \H$, $\param = [\A \,\Vert\, \B] \in \R^{(\kappa+p)\times n}$, and $\Param$ is a convex subset of $\R^{(\kappa+p)\times n}\times \R^{q\times \kappa}$. 
We have added a $L_{2}$-regularization term for $\A$ and $ \bgamma$ with coefficient $\lambda\ge 0$, which will play a crucial role in well-conditioning \eqref{eq:SMF_feat_CALE0}. 

For solving \eqref{eq:SMF_feat_CALE0}, we propose to use the LGPD algorithm \eqref{eq:LRPGD_iterate_main}:  \textit{We iterate gradient descent followed by projecting onto the convex constraint set $\Param$ of the combined factor $[\param, \bgamma]$ and then perform rank-$r$ projection of the first factor $\param=[\A \, \Vert\, \B]$ via truncated SVD until convergence.} Once we have a solution $[\param^{\star},\bgamma^{\star}]$ to \eqref{eq:SMF_feat_CALE0}, we can use SVD of $\param^{\star}$ to obtain a solution to \eqref{eq:SMF_main}. Let $\param^{\star} = \U \bSigma \V^{T}$ denote the SVD of $\param$. Since $\rank(\param^{\star})\le r$, $\bSigma$ is an $r\times r$ diagonal matrix of singular values of $\param$. Then $\U\in \R^{(\kappa+p)\times r}$ and $\V\in \R^{n\times r}$ are semi-orthonormal matrices, that is, $\U^{T}\U = \V^{T}\V=\I_{r}$. Then since $\param^{\star}=[(\Beta^{\star})^{T} \, \Vert\, \W^{\star}] \H^{\star}$, we can take $\H^{\star}=\bSigma^{1/2}\V^{T}$ and  $[(\Beta^{\star})^{T} \, \Vert\,\W^{\star}]= \U \bSigma^{1/2}$. 

We summarize this approach of solving \eqref{eq:SMF_main} for SMF-$\H$ in Algorithm \ref{alg:SMF_PGD}. Here, $\textup{SVD}_{r}$ denotes rank-$r$ truncated SVD and the projection operators $\Pi_{\Param}$ and $\Pi_{r}$ are defined in Subsection \ref{subsection:notation}.

\begin{algorithm}[h!]
	\caption{Lifted PGD for SMF}
	\label{alg:SMF_PGD}
	\begin{algorithmic}
		\State {\bfseries Input:} $\X_{\textup{data}}\in \R^{p\times n}$ ;\,\,$\X'_{\textup{aux}}\in \R^{q\times n}$ (auxiliary features);\,\, $\Y_{\textup{label}}\in \{0,1,\dots,\kappa\}^{n}$ 
		\State {\bfseries Parameters:} $\tau>0$ (stepsize);\,\, $N\in \mathbb{N}$ (iterations); $r\ge 1$ (rank); $\lambda\ge 0$ ($L_{2}$-reg. param.)
		\State {\bfseries Constraints:} Convex  $\Param\subseteq \R^{(\kappa+p)\times n}\times \R^{q\times \kappa}$ for SMF-$\H$, $\Param\subseteq \R^{p\times (\kappa+n)}\times  \R^{q\times \kappa}$ for SMF-$\W$; 
		\State Initialize $\W_{0}\in \R^{p\times r}$, $\H_{0}\in \R^{r\times n}$, $\Beta_{0}\in \R^{r\times \kappa}$, $\bgamma_{0}\in \R^{q\times \kappa}$
		\State $\begin{cases}
			\param_{0}\leftarrow [\Beta_{0}^{T}\H_{0} \parallel \W_{0}\H_{0}]\in \R^{(\kappa+p)\times n} & \hspace{-0.1cm} \textup{($\triangleright$ for SMF-$\H$)} \\
			\param_{0}\leftarrow [\W_{0}\Beta_{0}, \W_{0}\H_{0}]\in \R^{p\times (\kappa+n)} &
			\hspace{-0.1cm}  \textup{($\triangleright$ for SMF-$\W$)}
		\end{cases}$
		\For{$k=1$ {\bfseries to} $N$} 
		\State $\param_{k}\leftarrow \Pi_{r}\left( \Pi_{\Param} \left( \param_{k-1} - \tau \nabla_{\param}  F(\param_{k-1}, \bgamma_{k-1})  \right)\right) $ \qquad \textup{($\triangleright$ See Appendix \ref{sec:gradient_computation} for computation)}
		\State $\bgamma_{k} \leftarrow  \bgamma_{k-1} - \tau \nabla _{\bgamma}  F(\param_{k-1}, \bgamma_{k-1}) $
		\EndFor
		\State  $\param_{N}=\U \bSigma \V^{T} $ \qquad ($\triangleright$ rank-$r$ SVD)
		\State 
		$\begin{cases}
			[\Beta_{N}^{T} \parallel \W_{N}] \leftarrow \U \bSigma^{1/2},\, \H_{N} \leftarrow(\bSigma)^{1/2}\V^{T} & \hspace{-0.2cm}\textup{($\triangleright$ SMF-$\H$)}\\
			\W_{N} \leftarrow \U,\, [\Beta_{N},\,\, \H_{N}] \leftarrow \bSigma \V^{T} &\hspace{-0.2cm} \textup{($\triangleright$ SMF-$\W$)}
		\end{cases}$
		
		\State \textbf{Output:} $(\W_{N}, \H_{N}, \Beta_{N}, \bgamma_{N})$
	\end{algorithmic}
\end{algorithm}

As for SMF-$\W$, we can rewrite \eqref{eq:SMF_main} with additional $L_{2}$-regularizer for $\A=\W\Beta$ and $\bgamma$ as 
\begin{align}\label{eq:SMF_filt_CALE0}
	\min_{\substack{[\param, \bgamma] \in \Param \\  \rank(\param)\le r}}	    F\left(\param,\, \bgamma \right) 
	=   \sum_{i=1}^{n} \ell(y_{i}, \A^{T} \x_{i} + \bgamma^{T}\x_{i}' )  +   \xi \lVert  \X_{\textup{data}}  -\B\rVert_{F}^{2} + \lambda \left( \lVert \A \rVert_{F}^{2} + \lVert \bgamma \rVert_{F}^{2}  \right),  
\end{align}
where $\param=[\A,\B]=\W[\Beta,\H]\in \R^{p\times (\kappa+n)}$ and $\Param\in \R^{p\times (\kappa+n)}\times \R^{q\times \kappa}$ is a convex set. Algorithm \ref{alg:SMF_PGD} for SMF-$\W$ follows similar reasoning as before with the reformulation above.

By using randomized truncated SVD for the efficient low-rank projection in Algorithm \ref{alg:SMF_PGD}, the per-iteration complexity is $O(pn\min(n,p))$, while that for the nonconvex algorithm is $O((pr+q)n)$. While the LPGD algorithm is in general more expensive per iteration than the nonconvex method, the iteration complexity is only $O(\log \epsilon^{-1})$ thanks to the exponential convergence to the global optimum (will be discussed in Theorem \ref{thm:SMF_LPGD}). To our best knowledge, the nonconvex algorithm for SMF does not have any guarantee to converge to a global optimum, and the iteration complexity of the nonconvex SMF method to reach an $\epsilon$-stationary point is at best $O(\epsilon^{-1})$ using standard analysis. Hence for $\epsilon$ small enough, Algorithm \ref{alg:SMF_PGD} achieves an $\epsilon$-accurate global optimum for SMF with a total computational cost comparable to a nonconvex SMF algorithm to achieve an $\epsilon$-stationary point.


\section{Global convergence guarantee} 
We have discussed that one can cast the SMF problem \eqref{eq:SMF_main} as the following `factored estimation problem' $\min_{\T,\S, \bgamma} f(\T \S^{T},\bgamma)$. Note that such problems generally do not have a unique minimizer due to the `rotation invariance'. Namely, let $\bR$ be any $r\times r$ orthonormal (rotation) matrix (i.e., $\bR^{T}\bR=\bR \bR^{T}=\I_{r}$). Then $f((\T\bR)(\S\bR)^{T}, \bgamma) =  f(\T\bR \bR^{T} \S^{T} , \bgamma) =   f(\T\S^{T} , \bgamma)$. Hence the best one is to obtain parameters up to rotation that globally minimize the objective value. 
Our main result, Theorem \ref{thm:SMF_LPGD}, establishes that this can be achieved by Algorithm \ref{alg:SMF_PGD} at an exponential rate. 
First, we introduce the following technical assumptions (\ref{assumption:A2}-\ref{assumption:A4}). 

\begin{assumption}(Bounded activation)\label{assumption:A2}
	The activation $\a\in \R^{\kappa}$ defined in \eqref{eq:activation} assumes bounded norm, i.e., $\lVert \a \rVert\le M$ for some constant $M\in (0,\infty)$.
\end{assumption}

\begin{assumption}(Bounded eigenvalues of covariance matrix)\label{assumption:A3}
	Denote $\bPhi=[\bphi_{1},\dots,\bphi_{n}]\in \R^{(p+q)\times n}$, where $\bphi_{i} = [\x_{i} \parallel \x_{i}'] \in \R^{p+q}$ (so $\bPhi=[\X_{\textup{data}} \parallel  \X_{\textup{aux}}]$), where $\X_{\textup{aux}}=[\x_{1}',\dots,\x_{n}']$.  Then, there exist constants $\delta^{-},\delta^{+}>0$ such that for all $n\ge 1$, 
	\begin{align}
		\delta^{-} \le \lambda_{\min}( n^{-1} \bPhi \bPhi^{T}) \le \lambda_{\max}( n^{-1}\bPhi \bPhi^{T}) \le \delta^{+}.
	\end{align}
\end{assumption}

\begin{assumption}(Bounded stiffness and eigenvalues of observed information)\label{assumption:A4}
	The score function $h:\R\rightarrow [0,\infty)$  is twice continuously differentiable. Further, let observed information $\ddot{\H}(y,\a):=\nabla_{\a}\nabla_{\a^{T}}\ell(y,\a)$ for $y$ and $\a$.
	Then, for the constant $M>0$ in Assumption \ref{assumption:A2}, there are constants $\gamma_{\max}, \alpha^{-},\alpha^{+}>0$ s.t. $\gamma_{\max}:=\sup_{\lVert \a \rVert \le M} \max_{1\le s \le  n}\, \lVert \nabla_{\a} \ell(y_{s},\a) \rVert_{\infty}$ and 
	\begin{align}
		&	\alpha^{-}:=\inf_{\lVert \a \rVert \le M} \min_{1\le s \le  n}\, \lambda_{\min}(\ddot{\H}(y_{s},\a)), \quad \alpha^{+}:=\sup_{\lVert \a \rVert \le M} \max_{1\le s \le  n}\, \lambda_{\max}(\ddot{\H}(y_{s},\a)).
	\end{align}
\end{assumption}
Assumption \ref{assumption:A2} limits the norm of the activation $\a$ as an input for the classification model in \eqref{eq:SMF_main} is bounded. This is standard in the literature (see, e.g., \cite{negahban2011estimation, yaskov2016controlling, lecue2017sparse}) in order to uniformly bound the eigenvalues of the Hessian of the (multinomial) logistic regression model. Assumption \ref{assumption:A3} introduces uniform bounds on the eigenvalues of the covariance matrix. Assumption \ref{assumption:A4} introduces uniform bounds on the eigenvalues of the $\kappa\times \kappa$ observed information as well as the first derivative of the predictive probability distribution (see \cite{bohning1992multinomial} and Sec. \ref{sec:thm_proofs} in Appendix for more details). 
Under Assumption \ref{assumption:A2} and the multinomial logistic regression model $h(\cdot)=\exp(\cdot)$, one can derive Assumption \ref{assumption:A4} with a simple expression for the bounds $\alpha^{\pm}$, as discussed in the following remark.

\begin{remark}[Multinomial Logistic Classifier]
	\label{rmk:MNL_constants}
	Let $\ell$ denote the negative log-likelihood function in \eqref{eq:loglikelihood}, where we take the multinomial logistic model with the score function $h(\cdot)=\exp(\cdot)$. Denote $(\dot{h}_{1},\dots,\dot{h}_{\kappa}):=\nabla_{\a} \ell(y,\a)$ and  $\ddot{\H}(y,\a):=\nabla_{\a} \nabla_{\a^{T}} \ell(y,\a)$. Then in this special case,  we have $\dot{h}_{j}(y,\a)= g_{j}(\a) - \mathbf{1}(y=j)$ and $\ddot{H}(y,\a)_{i,j} = g_{i}(\a) \left(\mathbf{1}(i=j)-g_{j}(\a) \right)$ (See \eqref{eq:MNL_h_def_appendix} and \eqref{eq:MNL_H_def_appendix} in Appendix). 
	Under Assumption \ref{assumption:A2}, according to Lemma \ref{lem:MNL}, we can take $\gamma_{\max}=1+\frac{ e^{M} }{1+ e^{M} + (\kappa-1) e^{-M}}\le 2$, $\alpha^{-} = \frac{e^{-M}}{1+e^{-M}+(\kappa-1) e^{M}},$ and $ \alpha^{+} = \frac{e^{M} \left(1+2(\kappa-1)e^{M} \right) }{\left( 1+e^{M}+(\kappa-1) e^{-M} \right)^{2}}$.
	For binary classification, $\alpha^{+}\le 1/4$. 
\end{remark}

Now define the following quantities: \begin{align}\label{eq:thm1_condition_numbers}
	\mu:=\begin{cases}
		\min(2\xi,\,2\lambda+n\delta^{-}\alpha^{-}) \\
		\min(2\xi,\,2\lambda)
	\end{cases},
	\,\,  
	L&:=\begin{cases}
		\max(2\xi,\, 2\lambda+ n \delta^{+}\alpha^{+})  & \textup{for SMF-}\W \\
		\max( 2\xi, \, 2\lambda+\alpha^{+}) & \textup{for SMF-$\H$}
	\end{cases}.
\end{align}

Now, we state a special case of our first main result, specifically when the model is `correctly specified', allowing the rank-$r$ SMF model to effectively account for the observed data. This implies the existence of a `low-rank stationary point' of $F$, as also demonstrated in \cite{wang2017unified}. However, we also handle the general case in Appendix (see Theorem \ref{thm:SMF_LPGD_full}).

\begin{theorem}(Exponential convergence)\label{thm:SMF_LPGD}
	Let $\bZ_{t}:=[\param_{t}, \bgamma_{t}]$ denote the iterates of Algorithm \ref{alg:SMF_PGD}. Assume \ref{assumption:A2}-\ref{assumption:A4} hold. Let $\mu$ and $L$ be as in \eqref{eq:thm1_condition_numbers}, fix  $\tau\in ( \frac{1}{2\mu}, \frac{3}{2L})$, and let $\rho:=2(1-\tau\mu) \in (0,1)$. Suppose $L/\mu < 3$ and let $\bZ^{*}=[\param^{*}, \bgamma^{*}]$ be any stationary point of $F$ over $\Param$ s.t. $\rank(\param^{*})\le r$. Then $\bZ^{*}$ is the unique global minimizer of $F$ among all $\bZ=[\param,\bgamma]$ with $\rank(\param)\le r$. Moreover, $\lVert \bZ_{t} - \bZ^{*}  \rVert_{F}\le  \rho^{t}  \, \lVert  \bZ_{0} - \bZ^{*}\rVert_{F}$ for $t\ge 1$. 
\end{theorem}

In the statement above, we write $\lVert \bZ \rVert_{F}^{2}= \lVert [\param, \bgamma] \rVert_{F}^{2}:=\lVert \param \rVert_{F}^{2} + \lVert \bgamma \rVert_{F}^{2}$. Note that we may view the ratio $L/\mu$ that appears in Theorem \ref{thm:SMF_LPGD} as the condition number of the SMF problem in \eqref{eq:SMF_main}, whereas the ratio $L^{*}/\mu^{*}$ for $\mu^{*}:=\delta^{-}\alpha^{-}$ and $L^{*}:=\delta^{+}\alpha^{+}$ as the condition number for the multinomial classification problem. These two condition numbers are closely related. First, note that for any given $\mu^{*}, L^{*}$ and sample size $n$, we can always make $L/\mu<3$ by choosing sufficiently large $\xi$ and $\lambda$ so that Theorem \ref{thm:SMF_LPGD} holds. However, using large $L_{2}$-regularization parameter $\lambda$ may perturb the original objective in \eqref{eq:SMF_main} too much that the converged solution may not be close to the optimal solution. Hence, we may want to take $\lambda$ as small as possible. Setting $\lambda=0$ leads to 
\begin{align}\label{eq:thm_SMF_LGPD_cond_filt1}
	\frac{L}{\mu}<3,\, \lambda=0	 \,\, \Longleftrightarrow \,\, \begin{cases}
		0 < \frac{L^{*}}{\mu^{*}} < 3,\, \frac{L^{*}}{6}< \frac{\xi}{n} < \frac{3\mu^{*}}{2} & \textup{for SMF-}\W \\
		\frac{\max( 2\xi, \alpha^{+})}{ \min(2\xi,\, 0)} \,\, < \,\, 3 & \textup{for SMF-$\H$}. 
	\end{cases}
\end{align}
For SMF-$\W$, if  the multinomial classification problem is  well-conditioned ($L^{*}/\mu^{*}<3$) and the ratio $\xi/n$ is in the above interval, then SMF-$\W$ enjoys exponential convergence in Theorem \ref{thm:SMF_LPGD}. However, the condition for SMF-$\H$ in \eqref{eq:thm_SMF_LGPD_cond_filt1} is violated, so $L_{2}$-regularization is necessary for guaranteeing exponential convergence of SMF-$\H$.    

The proof of Theorem \ref{thm:SMF_LPGD} involves two steps: (1) We establish a general exponential convergence result for the general LPGD algorithm \eqref{eq:LRPGD_iterate_main} in Theorem \ref{thm:CALE_LPGD} in Appendix. (2) We compute the Hessian eigenvalues of the SMF objectives \eqref{eq:SMF_feat_CALE0}-\eqref{eq:SMF_filt_CALE0} and apply the result to obtain Theorem \ref{thm:SMF_LPGD}. The proof contains two challenges: first, the low-rank projection in \eqref{eq:LRPGD_iterate_main} is not non-expansive in general. To overcome this, we show that the iterates closely approximate certain `auxiliary iterates' which exhibit exponential convergence towards the global optimum. Secondly, the second-order analysis is highly non-trivial since the SMF problem \eqref{eq:SMF_main} has a total of four unknown matrix factors that are intertwined through the joint multi-class classification and matrix factorization tasks. See Appendix \ref{sec:thm_proofs} for the details.



\section{Statistical estimation guarantee}
\label{subsection:statistical_estimation}


In this section, we formulate a generative model for SMF \eqref{eq:SMF_main} and state statistical parameter estimation guarantee. Fix dimensions $p\gg q$, and let $n\ge 1$ be possibly growing sample size, and fix unknown true parameters $\B^{\star}\in \R^{p\times n},\, \C^{\star}\in \R^{q\times n},\,    \bgamma^{\star}\in \R^{q\times \kappa}$. In addition, fix  $\A^{\star}\in \R^{\kappa\times n}$ for SMF-$\H$ and $\A^{\star}\in \R^{p\times \kappa}$ for SMF-$\W$.
Now suppose that class label, data, and auxiliary features are drawn i.i.d. according to the following joint distribution:
\begin{align}\label{eq:SMF_generative}
	\begin{cases}
		&\x_{i} = \B^{\star}[:,i] +\beps_{i}, \quad \x_{i}'=\C^{\star}[:,i]+\beps_{i}', \\ &y_{i}\,|\, \x_{i}, \x_{i}' \sim \text{Multinomial}\big(1, \g\left( \a_{i} \right) \big),
		\\
		&\a_{i} = \begin{cases}
			\A^{\star}[:,i] + (\bgamma^{\star})^{T} \x_{i}' & \textit{SMF-$\H$},\\
			(\A^{\star})^{T}\x_{i} + (\bgamma^{\star})^{T} \x_{i}'  & \textit{SMF-$\W$},
		\end{cases}
	\end{cases}
	,\quad 
	\begin{cases}
		\rank([\A^{\star}\parallel \B^{\star}]) \le r  & \textit{for SMF-$\H$},\\
		\rank([\A^{\star},\B^{\star}])\le r & \textit{for SMF-$\W$}.
	\end{cases}
\end{align}
where each $\beps_{i}$ (resp., $\beps_{i}'$) are  $p\times 1$ (resp., $q\times 1$) vector of i.i.d. mean zero Gaussian entries with variance $\sigma^{2}$ (resp., $(\sigma')^{2}$). We call the above the \textit{generative SMF model}. In what follows, we will assume that the noise levels $\sigma$ and $\sigma'$ are known and focus on estimating the four-parameter matrices.

The ($L_{2}$-regularized) negative log-likelihood of observing triples $(y_{i},\x_{i},\x_{i}')$ for $i=1,\dots, n$ is given as $\L_{n} := F(\A,\B,\bgamma)  + \frac{1}{2(\sigma')^{2}} \lVert \X_{\textup{aux}}-\C \rVert_{F}^{2} + c$,	where $c$ is a constant and $F$ is as in \eqref{eq:SMF_feat_CALE0} or \eqref{eq:SMF_filt_CALE0} depending on the activation type with tuning parameter  $\xi=\frac{1}{2\sigma^{2}}$. The $L_{2}$ regularizer in $F$ can be understood as Gaussian prior for the parameters and interpreting the right-hand side above as the negative logarithm of the posterior distribution function (up to a constant) in a Bayesian framework. 
Note that the problem of estimating $\A$ and $\B$ are coupled due to the low-rank model assumption in \eqref{eq:SMF_generative}, while the problem of estimating $\C$ is standard and separable, so it is not of our interest. The joint estimation problem for $[\A,\B,\bgamma]$ is equivalent to the corresponding SMF problem \eqref{eq:SMF_main} with tuning parameter $\xi=(2\sigma^{2})^{-1}$. This and Theorem \ref{thm:SMF_LPGD} motivate us to estimate the true parameters $\A^{\star}, \B^{\star}$, and $\bgamma^{\star}$ by the output of Algorithm \ref{alg:SMF_PGD} with $\xi=(2\sigma^{2})^{-1} $ for $O(\log n)$ iterations.

Now we give the second main result. 
Roughly speaking, it states that the estimated parameter $\bZ_{t}$ is within the true parameter $\bZ^{\star}=[\A^{\star}, \B^{\star},\bgamma^{\star}]$ within $O(\log n/\sqrt{n})$ with high probability, provided that the noise variance $\sigma^{2}$ is small enough and the SMF objective \eqref{eq:SMF_feat_CALE0}-\eqref{eq:SMF_filt_CALE0} is well-conditioned.

\begin{theorem}(Statistical estimation for SMF)\label{thm:SMF_LPGD_STAT}
	Assume the model \eqref{eq:SMF_generative} with fixed $p$. Suppose Assumptions \ref{assumption:A2}-\ref{assumption:A4} hold. Let $\mu,L$ be as in \eqref{eq:thm1_condition_numbers},  $\rho:=2(1-\tau\mu)$ and $c=O(1)$ if $\bZ^{\star}-\tau \nabla_{\bZ} F (\bZ^{\star})\in \Param$ and $c=O(\sqrt{\min(p,n)})$ otherwise. Let $\bZ_{t}$ denote the iterates of Algorithm \ref{alg:SMF_PGD} with the tuning parameter $\xi= (2\sigma^{2})^{-1}$, $L_{2}$-regularization parameter $\lambda>0$, and stepsize $\tau\in ( \frac{1}{2\mu}, \frac{3}{2L})$. 
	Then following holds with probability at least $1-\frac{1}{n}$: For all $t\ge 1$ and $n\ge 1$, $\lVert \bZ_{t} - \bZ^{\star} \rVert_{F} -  \rho^{t}  \, \lVert  \bZ_{0} - \bZ^{\star}\rVert_{F} \le  c\frac{(\sqrt{n}\log n+\lambda)}{\mu } $,
	provided  $L/\mu<3$. Furthermore, $c\frac{(\sqrt{n}\log n+\lambda)}{\mu }$ is $O(\log n/\sqrt{n})$ if  $\bZ^{\star}-\tau \nabla_{\bZ} F (\bZ^{\star})\in \Param$ and $\sigma^{2}=O(1/n)$.
\end{theorem}

We remark that Theorem \ref{thm:SMF_LPGD_STAT} implies that \textit{SMF-$\H$ is statistically more robust than SMF-$\W$}. Namely, in order to have an arbitrarily accurate estimate with high probability, one needs to have $\mu \gg \sqrt{n}\log n$. Combining with the expression in \eqref{eq:thm1_condition_numbers} and the well-conditioning assumption  $L/\mu<3$, one needs to require $\xi =\Omega(n)$, hence small noise variance $\sigma^{2}=O(1/n)$ for SMF-$\W$. However, for SMF-$\H$, this is guaranteed whenever $\sigma^{2}=o(1/(\sqrt{n}\log n))$ and  $\lambda\approx \mu$.

\section{Simulation and Numerical Validation}

\begin{figure}[h!]
	\centering
	\includegraphics[width=1\linewidth]{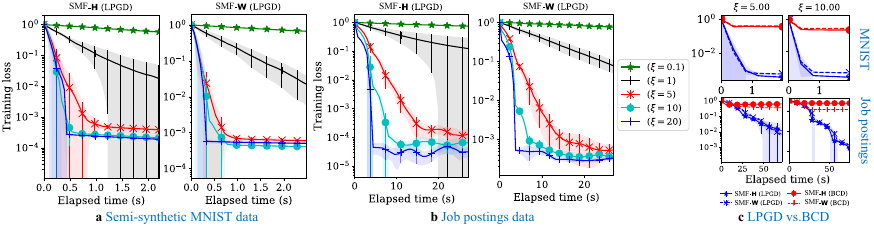} 
	\vspace{-0.5cm}
	\caption{(\textbf{a}-\textbf{b}) Training loss vs. elapsed CPU time for Algorithm \ref{alg:SMF_PGD} (with binary logistic classifier) on the semi-synthetic MNIST and Job postings datasets for several values of $\xi$ in log scale.  Average training loss over ten runs and the shades representing the standard deviation shown. (\textbf{c}) Comparison between LPGD (Algorithm \ref{alg:SMF_PGD}) and BCD algorithms for SMF.}
	\label{fig:benchmark_MNIST}
\end{figure}

We numerically verify Theorem \ref{thm:SMF_LPGD} on a semi-synthetic dataset generated by using MNIST image dataset  \cite{lecun-mnisthandwrittendigit-2010} ($p=28^{2}=784$, $q=0$, $n=500$, $\kappa=1$) and a text dataset named `Real / Fake Job Posting Prediction' \cite{fakejob} ($p =2840, q = 72, n = 17880, \kappa = 1$). 
Details about these datasets are in Sec. \ref{sec:experimental_details} in Appendix.\footnote{We provide our implementation of Algorithm \ref{alg:SMF_PGD} in our code repository \url{https://github.com/ljw9510/SMF/tree/main}.} We used Algorithm \ref{alg:SMF_PGD} with rank $r=2$ for MNIST and $r=20$ for job postings datasets. For all experiments, $\lambda=2$ and stepsize $\tau=0.01$ were used.

We validate the theoretical exponential convergence results of our LPGD algorithm (Algorithm \ref{alg:SMF_PGD}) in Figure \ref{fig:benchmark_MNIST}. Note that the convexity and smoothness parameters $\mu$ and $L$ in  Theorem \ref{thm:SMF_LPGD} are difficult to compute exactly. In practice, cross-validation of hyperparameters is usually employed. For $\xi\in \{ 0.1, 1, 5,10,20 \}$ in Figure \ref{fig:benchmark_MNIST}, we indeed observe exponential decay of training loss as dictated by our theoretical results for Algorithm \ref{alg:SMF_PGD}.  We also observe that the exponential rate of decay in training loss increases as $\xi$ increases. According to Theorem \ref{thm:SMF_LPGD}, the contraction coefficient is $\rho=(1-\tau \mu)$, which decreases in $\xi$ since $\mu$ increases in $\xi$ (see \eqref{eq:thm1_condition_numbers}). The decay for large $\xi\in \{10,20 \}$ seems even superexponential. Furthermore, \ref{fig:benchmark_MNIST}\textbf{c} shows that our LPGD algorithm converges significantly faster than BCD for training both SMF-$\H$ and SMF-$\W$ models at $\xi\in \{5,10\}$ (other values of $\xi$ omitted). 

\section{Application: Microarray Analysis for Cancer Classification}

We apply the proposed methods to two datasets from the Curated Microarray Database (CuMiDa) \cite{Feltes2019}. CuMiDa provides well-preprocessed microarray data for various cancer types for various machine-learning approaches. One consists of 54,676 gene expressions from 51 subjects with binary labels indicating pancreatic cancer; Another we use has 35,982 gene expressions from 289 subjects with binary labels indicating breast cancer. The primary purpose of the analysis is to classify cancer patients solely based on their gene expression. We compare the accuracies of the proposed methods -- SMF-$\W$ and SMF-$\H$ with a binary logistic classifier trained using Algorithm \ref{alg:SMF_PGD} -- against the following benchmark algorithms:  SMF-$\W$ and SMF-$\H$ trained using BCD; 
1-dimensional seven-layer Convolutional Neural Networks (CNN); three-layer Feed-Forward Neural Networks (FFNN); Naive Bayes (NB); Support Vector Machine (SVM); Random Forest (RF); Logistic Regression with Matrix Factorization by truncated SVD (MF-LR). For the last benchmark method MF-LR, we use rank-$r$ SVD to factorize $\X_{\textup{train}} \approx \U \Sigma\V^{T}$ and take $\W=\U \Sigma$ and $\H=\V^{T}$. For testing, we use $\W^{T}\X_{\textup{test}}$ as input to logistic regression for both filter and feature methods since $\lVert \X_{\textup{test}}-\W \H_{\textup{test}} \rVert_{F}$ is minimized when $\H_{\textup{test}}=(\W^{T}\W)^{-1}\W^{T}\X_{\textup{test}}=\W^{T}\X_{\textup{test}}$ with orthogonal $\W$.

\begin{figure}[h!]
	\centering
	\includegraphics[width=1\linewidth]{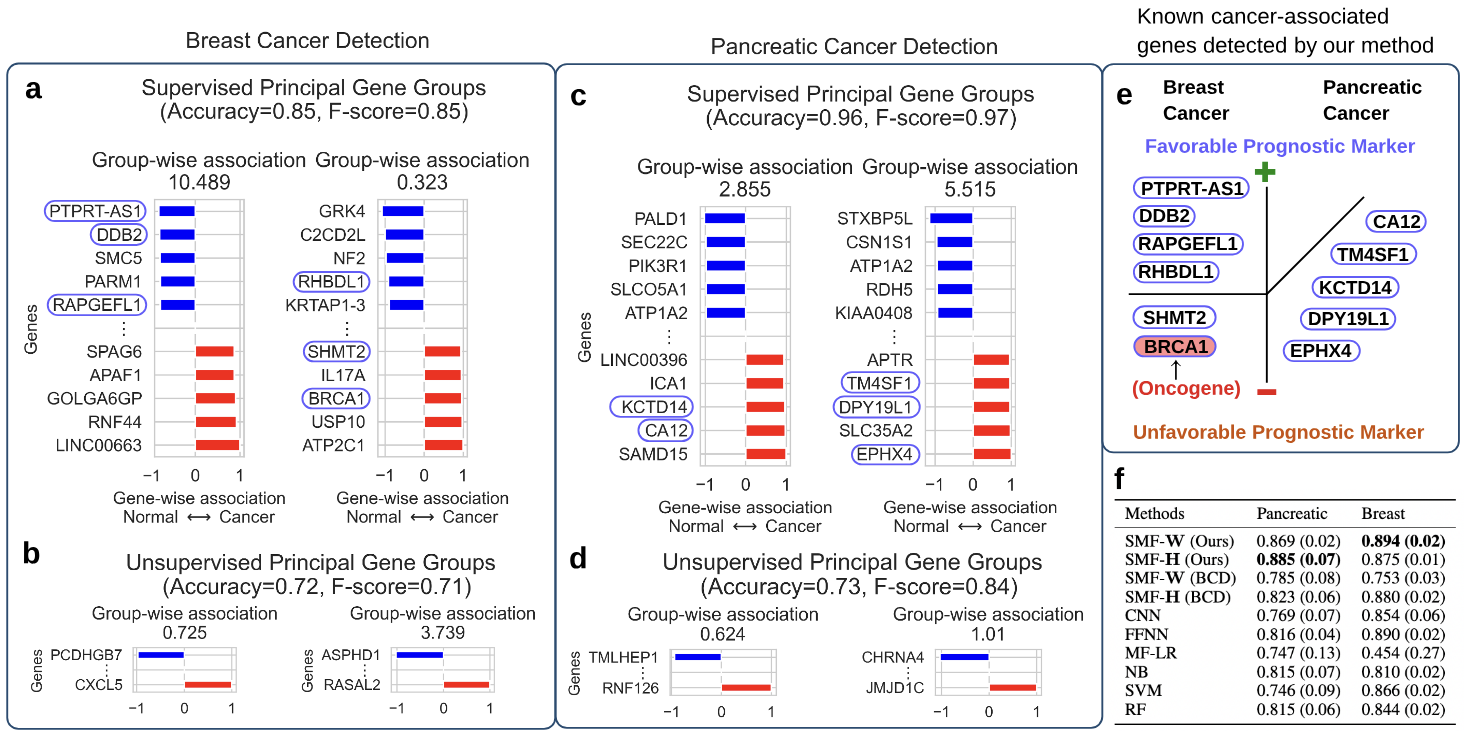} 
	\vspace{-0.5cm}
	\caption{ 
		(\textbf{a-b}) Two selected supervised/unsupervised principal gene groups (low-dimensional compression of genes) learned by rank-16 SMF-$\W$/SVD and their associated logistic regression coefficients for breast cancer detection. (\textbf{c}-\textbf{d}) Similar to \textbf{a}-\textbf{b} learned by rank-2 SMF-$\W$/SVD for pancreatic cancer detection. 
		\textbf{(e)} Blue-circled genes within each gene group of extreme coefficients coincide with known prognostic markers (for pancreatic cancer) and oncogene (for breast cancer). (\textbf{f}) Average classification accuracies and their standard deviations (in parenthesis) for various methods on two cancer microarray datasets over five-fold cross-validation. The highest-performing instances are marked in bold. 
	}
	\label{fig:pancreatic_cancer}
\end{figure}

We normalize gene expression for stable matrix factorization and interpretability of regression coefficients. We split each data into 50\% of the training set and 50\% of the test set and repeat the comparison procedure 5 times. A scree plot is used to determine the rank $r$. Other parameters are chosen through 5-fold cross-validation ($\xi\in \{0.1, 1, 10\}$ and $\lambda\in \{0.1,1,10\}$), and the algorithms are repeated in 1,000 iterations or until convergence. 
As can be seen in the table in Figure \ref{fig:pancreatic_cancer}\textbf{a},  the proposed methods show the best performance for both types of cancers. 

An important advantage of SMF methods is that they provide interpretable results in the form of `supervised factors'. Each supervised factor consists of a latent factor and the associated regression coefficient. That is, once we train the SMF model (for $\kappa=1$) and learn factor matrix $\W=[\w_{1},\dots,\w_{r}]\in \R^{p\times r}$ and vector of regression coefficients $\Beta=[\beta_{1},\dots,\beta_{r}]\in \R^{1\times r}$, each column $\w_{j}$ of $\W$ describes a latent factor 
and the corresponding regression coefficient $\beta_{j}$ tells us how $\w_{j}$ is associated with class labels. The pairs $(\w_{j},\beta_{j})$, which form supervised latent factors, provide insights into how the trained SMF model perceives the classification task. See Fig. \ref{fig:flowchart} for illustration.

In the context of microarray analysis for cancer research, each $\w_{j}$ corresponds to a weighted group of genes (which we call a `principal gene group') and $\beta_{j}$ represents the strength of its association with cancer. SMF learns supervised gene groups (Fig. \ref{fig:pancreatic_cancer}\textbf{a, c}) with significantly higher classification accuracy than the unsupervised gene groups (Fig. \ref{fig:pancreatic_cancer}\textbf{b, d}). In Fig. \ref{fig:pancreatic_cancer}\textbf{a, c}, both gene groups (consisting of $p$ genes) have positive regression coefficients, so they are positively associated with the log odds of the predictive probability of the corresponding cancer. 
Remarkably, our method detected the well-known oncogene BRCA1 of breast cancer and other various genes (in Fig. \ref{fig:pancreatic_cancer}\textbf{e}) that are known to be prognostic markers of breast/pancreatic cancer (see Human Protein Atlas \cite{sjostedt2020atlas}) in these groups of extreme coefficients (top five). The high classification accuracy suggests that the identified supervised principal gene groups may be associated with the occurrence of breast/pancreatic cancer.

\section{Conclusion and Limitations}

We propose an exponentially convergent algorithm for nonconvex SMF training using new lifting techniques. Our analysis demonstrates strong convergence and estimation guarantee. We compare the robustness of filter-based and feature-based SMF, finding that the former is computationally more robust while the latter is statistically more robust. The algorithm's exponential convergence is numerically verified. In cancer classification using microarray data analysis, our algorithm successfully identifies discriminative gene groups for various cancers and shows potential for identifying important gene groups as protein complexes or pathways in biomedical research. Our analysis framework can be extended to more complex classification models, such as combining a feed-forward deep neural network with a matrix factorization objective. While our convergence analysis holds in certain parameter regimes, we discuss them in detail. We have tested our method and convergence analysis on various real-world datasets but recommend further numerical verification on a wider range of datasets.


\bibliography{mybib}
\bibliographystyle{plain}

\newpage
\appendix
\onecolumn

\section{Gradient computation for executing the main algorithm}\label{sec:gradient_computation}

A straightforward computation shows (recall that $\param=[\A \parallel \B]$ for SMF-$\H$ and $\param=[\A, \B]$ for SMF-$\W$)
\begin{align}\label{eq:SMF_gradients}
	&\nabla_{\vect(\A)}  F - 2\lambda \vect(\A)=   \begin{cases} \sum_{s=1}^{n}\nabla_{\a} \ell(y_{s},\a_{s}) \otimes \I_{n}[:,s]  & \textup{for SMF-$\H$} \\
		\sum_{s=1}^{n}\nabla_{\a} \ell(y_{s},\a_{s}) \otimes \x_{s} & \textup{for SMF-$\W$,}
	\end{cases}
	\\
	&\nabla_{\B}  F = 2\xi (\B-\X_{\textup{data}}), \qquad \nabla_{\vect(\bgamma)}  F = \left(  \sum_{s=1}^{n} \nabla_{\a} \ell(y_{s},\a_{s}) \otimes \x_{s}' \right) + 2\lambda \vect(\bgamma),
\end{align}
where $\otimes$ denotes the Kronecker product. Here, we have $\nabla_{\a} \ell(y,\a)=(\dot{h}_{1},\dots,\dot{h}_{\kappa})$, where 
\begin{align}\label{eq:hdot_def}
	&\dot{h}_{j}:=\frac{h'(a_{j})}{1+\sum_{c=1}^{\kappa} h(a_{c})} - \mathbf{1}_{\{y=j \} }\frac{h'(a_{j})}{h(a_{j})}. 
\end{align}

\section{Generalized multinomial logistic regression}\label{sec:MLR}
In this section, we provide some background on a generalized multinomial logistic regression and record some useful computations. (See \cite{bohning1992multinomial} for backgrounds on multinomial logistic regression.)
Without loss of generality, we can assume that the $\kappa+1$ classes are the integers in $\{0,1,\dots, \kappa \}$. Say we have training examples $(\bphi(\x_{1}),y_{1}),\dots,(\bphi(\x_{n}),y_{n})$, where 
\begin{enumerate}[label={\textbf{$\bullet$}}]
	\item $\x_{1},\dots,\x_{n}$: Input data (e.g., collection of all medical records of each patient) 		
	\item $\bphi_{1}:=\bphi(\x_{1}),\dots,\bphi_{n}:=\bphi(\x_{n})\in \R^{p}:$ Features (e.g., some useful  information for each patient)
	\item $y_{1},\dots,y_{n}\in \{0, 1,\dots, \kappa \}$: $\kappa+1$ class labels (e.g., digits from 0 to 9). 
\end{enumerate}
The basic idea of multinomial logistic regression is to model the output $y$ as a discrete random variable  $Y$ with probability mass function  $\p=[p_{0}, p_{1},\dots,p_{\kappa}]$ that depends on the observed feature $\bphi(\x)$, score function $h:\R \rightarrow \R$ (strictly increasing, twice differentiable, and $h(0)=1$), and a matrix parameter $\W=[\w_{1},\dots,\w_{\kappa}]\in \R^{p\times \kappa}$ through the following relation: 
\begin{align}\label{eq:MLR}
	p_{0}= \frac{1}{1+\sum_{c=1}^{\kappa} h(\langle \bphi(\x), \w_{c} \rangle)},\quad 	p_{j}=\frac{h(\langle \bphi(\x), \w_{j} \rangle)}{1+\sum_{c=1}^{\kappa} h(\langle \bphi(\x), \w_{c} \rangle)},\quad \text{for $j=1,\dots, \kappa$}.
\end{align}
That is, given the feature vector $\bphi(\x)$, the probability $p_{i}$ of $\x$ having label $i$ is proportional to $h$ evaluated at the `linear activation' $\langle \bphi(\x), \w_{i} \rangle$ with the base category of class 0. Note that using $h(x)=\exp(x)$, the above multiclass classification model reduces to the classical multinomial logistic regression. In this case, the corresponding predictive probability distribution $\p$ is called  the \textit{softmax distribution} with activation $\mathbf{a}=[a_{1},\dots,a_{\kappa}]$ with $a_{i}=\langle \bphi(\x), \w_{i} \rangle$ for $i=1,\dots,\kappa$. Notice that this model has parameter vectors $\w_{1},\dots,\w_{\kappa}\in \R^{p}$, one for each of the $\kappa$ nonzero class labels.

Next, we derive the maximum log-likelihood formulation for finding optimal parameter $\W$ for the given training set $(\bphi_{i},y_{i})_{i=1,\dots,n}$. For each $1\le i \le n$, define the predictive probability mass function $\mathbf{p}_{i}= [p_{i0},p_{i1},\dots,p_{i\kappa}]$ using \eqref{eq:MLR} with $\bphi(\x)$ replaced by $\bphi_{i}$. We introduce the following matrix notations 
\begin{align}
	&
	\begin{matrix}
		\Y:=
		\begin{bmatrix}
			\mathbf{1}(y_{1}=1) & \cdots & \mathbf{1}(y_{1}=\kappa) \\ \vdots & & \vdots  \\ \mathbf{1}(y_{n}=1)  & \cdots & \mathbf{1}(y_{n}=\kappa)
		\end{bmatrix} 
		\\ \\ 
		\quad \in \{0,1\}^{n\times \kappa}
	\end{matrix}
	,\,\, 
	\begin{matrix}
		\bP:=
		\begin{bmatrix}
			p_{11} & \cdots & p_{1\kappa} \\ \vdots & & \vdots  \\ p_{n 1}  & \cdots & p_{n \kappa}
		\end{bmatrix} 
		\\ \\ 
		\quad \in [0,1]^{n\times \kappa}
	\end{matrix}
	\\
	&
	\begin{matrix}
		\bPhi:= 	
		\begin{bmatrix}
			\uparrow & & \uparrow \\
			\bphi(\x_{1})  & \cdots &  \bphi(\x_{n}) \\
			\downarrow &  &\downarrow
		\end{bmatrix}
		\\ \\
		\quad \in \R^{p\times n}
	\end{matrix}
	,\,\, 
	\begin{matrix}
		\W := 
		\begin{bmatrix}
			\uparrow & & \uparrow \\
			\w_{1}  & \cdots &  \w_{\kappa}\\
			\downarrow &  &\downarrow
		\end{bmatrix}
		.
		\\ \\
		\quad \in \R^{p\times \kappa}
	\end{matrix}
\end{align}
Note that the $s$th row of $\Y$ is a zero vector if and only if $y_{s}=0$. Similarly, since $p_{s0} = 1- (p_{s1}+\dots+p_{s\kappa})$, the corresponding row of $\bP$ determines its predictive probability distribution. Then the joint likelihood function of observing labels $(y_{1},\dots,y_{n})$ given input data $(\x_{1},\dots,\x_{n})$ under the above probabilistic model is 
\begin{align}
	L(y_{1},\dots,y_{n}\,;\, \W) = \P(Y_{1}=y_{1},\dots,Y_{n}=y_{n}\, ;\, \W)=  \prod_{s=1}^{n} \prod_{j=0}^{\kappa}  (p_{sj})^{\mathbf{1}(y_{s}=j)}.
\end{align}
Denote $\w_{0}=\mathbf{0}$. Then since $h(0)=1$ by definition,  we can conveniently write 
\begin{align}
	p_{sj}= \frac{h(\langle \bphi_{s}, \w_{j} \rangle)}{\sum_{c=0}^{\kappa} h(\langle \bphi_{s}, \w_{c} \rangle)} \quad \textup{for $s=1,\dots, n$ and $j=0,1,\dots,\kappa$}.
\end{align}

Now we can derive the negative log-likelihood  $\ell(\bPhi, \W)
:= - \sum_{s=1}^{n} \sum_{j=0}^{\kappa} \mathbf{1}(y_{s}=j) \log p_{sj}$ in a matrix form as follows: 
\begin{align}
	\ell(\bPhi, \W)
	& = \sum_{s=1}^{n}  \log \left( 1+\sum_{c=1}^{\kappa} h(\langle \bphi(\x_{s}), \w_{c} \rangle ) \right) - \sum_{s=1}^{n}  \sum_{j=0}^{\kappa} \mathbf{1}(y_{s}=j) \log h\left(  \langle \bphi(\x_{s}), \w_{j}\rangle \right)  \\
	&= \left( \sum_{s=1}^{n} \log \left( 1+\sum_{c=1}^{\kappa} h(\langle \bphi(\x_{s}), \w_{c} \rangle ) \right) \right)- \tr\left( \Y^{T} h(\bPhi^{T} \W) \right),
\end{align}
where $\tr(\cdot)$ denotes the trace operator. Then the maximum likelihood estimates $\hat{\W}$ is defined as the minimizer of the above loss function in $\W$ while fixing the feature matrix $\bPhi$. 

Both the maps $\W\mapsto \ell(\bPhi, \W)$ and $\bPhi\mapsto \ell(\bPhi, \W)$ are convex and we can compute their gradients as well as the Hessian explicitly as follows. For each $y\in \{0,1,\dots\kappa\}$, $\bphi\in \R^{p}$, and $\W\in \R^{p\times \kappa}$, define vector and matrix functions 
\begin{align}
	&\dot{\h}(y,\bphi, \W) :=(\dot{h}_{1},\dots, \dot{h}_{\kappa})^{T}\in \R^{\kappa\times 1}, \,\, \dot{h}_{j} := \frac{h'(\langle \bphi,\w_{j} \rangle)}{1+\sum_{c=1}^{\kappa} h(\langle \bphi,\w_{c} \rangle)} - \mathbf{1}(y=j)\frac{h'(\langle \bphi,\w_{j} \rangle)}{h(\langle \bphi,\w_{j} \rangle)} \label{eq:MNL_h_def_appendix}\\
	&\ddot{\H}(y,\bphi,\W)  := \left( \ddot{\H}_{ij}  \right)_{i,j} \in \R^{\kappa\times \kappa}, \\ &\ddot{\H}_{ij} =  \begin{matrix}  \frac{ h''(\langle \bphi,\w_{j} \rangle) \mathbf{1}(i=j)  }{  1+\sum_{c=1}^{\kappa} h(\langle \bphi,\w_{c} \rangle)  } - \frac{ h'(\langle \bphi,\w_{i} \rangle) h'(\langle \bphi,\w_{j} \rangle)   }{  \left( 1+\sum_{c=1}^{\kappa} h(\langle \bphi,\w_{c} \rangle) \right)^{2} }  
		- \mathbf{1}(y=i=j) \left( \frac{h''(\langle \bphi,\w_{j} \rangle)}{h(\langle \bphi,\w_{j} \rangle)} - \frac{\left( h'(\langle \bphi,\w_{j} \rangle) \right)^{2} }{\left( h(\langle \bphi,\w_{j} \rangle) \right)^{2}}   \right) \end{matrix}. \label{eq:MNL_H_def_appendix}
\end{align}
For each $\W=[\w_{1},\dots,\w_{\kappa}]\in \R^{p\times \kappa}$, let $\W^{\textup{vec}}:=[\w_{1}^{T},\dots, \w_{\kappa}^{T}]^{T}\in \R^{p \kappa}$ denote its vectorization.
Then a straightforward computation shows 
\begin{align}\label{eq:MNL_gradients_H}
	\nabla_{\vect(\W)}   \ell(\bPhi, \W) &= \sum_{s=1}^{n} \dot{\h}(y_{s},\bphi_{s}, \W) \otimes \bphi_{s}, \\
	\H:= \nabla_{\vect(\W)}  \nabla_{\vect(\W)^{T}}   \ell(\bPhi, \W) &= \sum_{s=1}^{n} \ddot{\H}(y_{s},\bphi_{s}, \W) \otimes \bphi_{s}\bphi_{s}^{T},
\end{align}
where $\otimes$ above denotes the Kronecker product. 	Recall that the eigenvalues of $\A\otimes \B$, where $\A$ and $\B$ are two square matrices, are given by $\lambda_{i}\mu_{j}$, where $\lambda_{i}$ and $\mu_{j}$ run over all eigenvalues of $\A$ and $\B$, respectively. Also, for two square matrices $\A,\B$ of the same size, write $\A \preceq \B$ if $v^{T}\A v\le v^{T} \B v$ for all unit vectors $v$. Then denoting $\lambda^{+}:=\max_{1\le s \le n}\lambda_{\max}(\ddot{\H}(y_{s},\bphi_{s}, \W) )$, 
\begin{align}
	\H \preceq \lambda^{+} \sum_{s=1}^{n} \I \otimes \bphi_{s} \bphi_{s}^{T} =  \lambda^{+}  \I \otimes \bPhi \bPhi^{T}. 
\end{align}
Similarly, $\lambda^{-}  \I \otimes \bPhi \bPhi^{T}\preceq \H$, where $\lambda^{-}$ denotes the minimum over all $\lambda_{\min}(\ddot{\H}(y_{s},\bphi_{s}, \W) )$. Hence we can deduce 
\begin{align}\label{eq:MNL_evals_bounds}	
	\lambda^{-}\lambda_{\min}\left( \bPhi \bPhi^{T} \right) 	 	\le \lambda_{\min}(\H) \le  \lambda_{\max}(\H)\le 	\lambda^{+}\lambda_{\max}\left( \bPhi \bPhi^{T} \right).
\end{align}

There are some particular cases worth noting. First, suppose binary classification case, $\kappa=1$. Then the Hessian $\H$ above reduces to 
\begin{align}
	\H = \sum_{s=1}^{n} \ddot{\H}_{11}(y_{s}, \bphi_{s},\W)  \bphi_{s}\bphi_{s}^{T}.
\end{align}
Second, let $h(x)=\exp(x)$ and consider the multinomial logistic regression case. Then $h=h'=h''$ so the above yields the following concise matrix expression
\begin{align}
	&\nabla_{\W} \, \ell(\bPhi, \W) = \bPhi (\bP - \Y) \in \R^{p\times \kappa},\qquad \nabla_{\bPhi} \, \ell(\bPhi, \W) = \W (\bP - \Y)^{T} \in \R^{p\times n}, \\
	& \H= \sum_{s=1}^{n} \begin{bmatrix}
		p_{s1}(1-p_{s1}) & - p_{s1}p_{s2} & \dots & -p_{s1}p_{s\kappa} \\
		-p_{s2}p_{s1} & p_{s2}(1-p_{s2}) & \dots & -p_{s2}p_{s\kappa} \\
		\vdots & \vdots & \ddots & \vdots \\
		-p_{s\kappa}p_{s1} & -p_{s\kappa}p_{s2} & \dots & p_{s\kappa}(1-p_{s\kappa}) 
	\end{bmatrix}
	\otimes \bphi_{s}\bphi_{s}^{T}.
\end{align}
Note that $\H$ in this case does not depend on $y_{s}$ for $s=1,\dots,n$.
The bounds on the eigenvalues depends on the range of linear activation $\langle \bphi_{i}, \w_{j} \rangle$ may take. For instance, if we restrict the norms of the input feature vector $\phi_{i}$ and parameter $\w_{j}$, then we can find a suitable positive uniform lower bound on the eigenvalues of $\H$.

\begin{lemma}\label{lem:MNL}
	Suppose $h(\cdot)=\exp(\cdot)$. Then 
	\begin{align}
		&\lambda_{\min}\left( \ddot{\H}(\bphi_{s},\W) \right)\ge \min_{1\le i \le \kappa} \frac{ \exp(\langle \bphi_{s}, \w_{i} \rangle) }{\left( 1+\sum_{c=1}^{\kappa} \exp(\langle \bphi_{s}, \w_{c} \rangle) \right)^{2} }, \\
		&\lambda_{\max}\left( \ddot{\H}(\bphi_{s},\W) \right)\le \max_{1\le i \le \kappa} \frac{ \exp(\langle \bphi_{s}, \w_{i} \rangle)  }{\left( 1+\sum_{c=1}^{\kappa} \exp(\langle \bphi_{s}, \w_{c} \rangle) \right)^{2} }\left(1+2\sum_{c=2}^{\kappa}  \exp(\langle \bphi_{s}, \w_{c} \rangle)  \right).
	\end{align}
\end{lemma}

\begin{proof}
	For the lower bound on the minimum eigenvalue, we note that
	\begin{align}
		\lambda_{\min}\left( \ddot{\H}(\bphi_{s},\W) \right) \ge \min_{1\le i\le \kappa}  \sum_{j=1}^{\kappa}  \ddot{H}_{ij} = \min_{1\le i \le \kappa} p_{si}p_{s0} =\min_{1\le i \le \kappa}  \frac{ \exp(\langle \bphi_{s}, \w_{i} \rangle) }{\left( 1+\sum_{c=1}^{\kappa} \exp(\langle \bphi_{s}, \w_{c} \rangle) \right)^{2} },
	\end{align}
	where the first inequality was shown in  \cite{amani2021ucb} using the fact that $\ddot{\H}(\bphi_{s},\W)$ is a diagonally dominant $M$-matrix (see \cite{tian2010inequalities}). The following equalities can be verified easily. 
	
	For the upper bound on the maximum eigenvalue, we use the Gershgorin circle theorem (see, e.g., \cite{horn2012matrix}) to bound 
	\begin{align}
		\lambda_{\max}\left( \ddot{\H}(\bphi_{s},\W) \right) &\le \max_{1\le i \le \kappa} \left( p_{si}(1-p_{si}) + \sum_{c\ne i} p_{si}p_{sc}\right) \le \max_{1\le i \le \kappa}  p_{si} \left(2-p_{s0}-2p_{si}\right).
	\end{align}
	Then simplifying the last expression gives the assertion.
\end{proof}

\section{Exponential convergence of low-rank PGD}

In  Section \ref{subsection:sketch} of the main manuscript, we outlined our key idea for solving the SMF problem \eqref{eq:SMF_main}, which involves  `double lifting` the nonconvex problem to a low-rank matrix estimation problem. 
In this section, we make this approach precise by considering an abstract form of optimization problems that generalizes the SMF problem  \eqref{eq:SMF_main}.

Fix a function $f:\R^{d_{1}\times d_{2}}\times \R^{d_{3}\times d_{4}}\rightarrow \R$, which takes the input of a $d_{1}\times d_{2}$ matrix and an augmented variable in $\R^{d_{3}\times d_{4}}$.  Consider the following  \textit{constrained and augmented   low-rank estimation}  (CALE) problem 
\begin{align}\label{eq:CALE}
	&\min_{\bZ=[\X, \bGamma] \in \subseteq  \R^{d_{1}\times d_{2}} \times \R^{d_{3}\times d_{4}}} \,f(\bZ),\qquad \textup{subject to $\bZ\in \Param$ and $\rank(\X)\le r$},
\end{align}
where $\Param$ is a convex subset of $\R^{d_{1}\times d_{2}}\times \R^{d_{3}\times d_{4}}$. Here, we seek to find a global minimizer $\bZ^{\star}=[\X^{\star}, \bGamma^{\star}]$ of the objective function $f$ over the convex set $\Param$, consisting of a low-rank  component $\X^{\star}\in \R^{d_{1}\times d_{2}}$ and an auxiliary variable $\bGamma^{\star}\in \R^{d_{3}\times d_{4}}$. In a statistical inference setting, the loss function $f=f_{n}$ may be based on $n$ noisy observations according to a probabilistic model, and the true parameter $\bZ^{*}$ to be estimated may approximately minimize $f$ over the constraint set $\Param$, with some statistical error $\eps(n)$ depending on the sample size $n$. In this case, a global minimizer $\bZ^{\star}\in \argmin_{\Param} f$ serves as an estimate of the true parameter $\bZ^{*}$.  
The matrix completion and low-rank matrix estimation problem \cite{meka2009guaranteed, recht2010guaranteed} can be considered as special cases of \eqref{eq:CALE} without constraint $\Param$ and the auxiliary variable $\bGamma$. 
This problem setting has been one of the most important research topics in the machine learning and statistics literature for the past few decades. More importantly for our purpose, we have seen in \eqref{eq:SMF_feat_CALE0} and \eqref{eq:SMF_filt_CALE0} in the main manuscript that both the feature- and filter-based SMF problems can be cast as the form of \eqref{eq:CALE} after some lifting and change of variables.

One can reformulate \eqref{eq:CALE} as the following nonconvex problem, where one parameterizes the low-rank matrix variable $\X$ with product $\U\V^{T}$ of two matrices, which we call the \textit{constrained and augmented factored estimation} (CAFE) problem:
\begin{align}\label{eq:CAFE}
	&\min_{ \U\in \R^{d_{1}\times r}, \V\in \R^{d_{2}\times r}, \bGamma\in \R^{d_{3}\times d_{4}} } \,f(\U\V^{T}, \bGamma) ,\qquad \textup{subject to $[\U\V^{T},\bGamma]\in \Param$}.
\end{align}
Note that a solution to \eqref{eq:CAFE} gives a solution to \eqref{eq:CALE}. Conversely, when considering \eqref{eq:CALE} without any constraint on the first matrix component, the singular value decomposition of the first matrix component easily demonstrates that a solution to \eqref{eq:CALE} is also a solution to \eqref{eq:CAFE}. Recently, there has been a surge of progress in global guarantees of solving the factored problem \eqref{eq:CAFE} using various nonconvex optimization methods \cite{jain2010guaranteed, jain2013low, zhao2015nonconvex, zheng2015convergent, tu2016low, park2017non, wang2017unified, park2016provable, park2018finding}. While most of the work considers \eqref{eq:CAFE} without the auxiliary variable and constraints, some studies consider specific types of convex constraints such as matrix norm bound. We consider $\Param$ to be a general convex constraint set.

It is common that the nonconvex factored problem \eqref{eq:CAFE} is introduced as a more efficient formulation of the convex problem \eqref{eq:CALE}. Interestingly, in the present work, we reformulate the four-factor nonconvex problem of SMF in \eqref{eq:SMF_main} as a three-factor nonconvex CAFE problem in \eqref{eq:CAFE} and then realize it as a single-factor convex CALE problem in \eqref{eq:CALE}. We illustrated this connection briefly in Section \ref{subsection:sketch} of the main manuscript.

In order to solve the CALE problem \eqref{eq:CALE}, consider the following \textit{low-rank projected gradient descent} (LPGD) algorithm: (see Remark \ref{rmk:projection} for more discussion on the use of projections $\Pi_{\Param}$ and $\Pi_{r}$)
\begin{align}\label{eq:LRPGD_iterate0}
	\bZ_{t} \leftarrow \Pi_{r}\left( \Pi_{\Param} \left(\bZ_{t-1} - \tau \nabla f(\bZ_{t-1}) \right) \right), 
\end{align}
where $\tau$ is a stepsize parameter, $\Pi_{\Param}$ denotes projection onto the convex constraint set $\Param\subseteq \R^{d_{1}\times d_{2}} \times \R^{d_{3}\times d_{4}}$, and $\Pi_{r}$ denotes the projection of the first matrix component onto matrices of rank at most $r$ in $\R^{d_{1}\times d_{2}}$. More precisely, let $\bZ = [\X,\bGamma]$. Then $\Pi_{r}(\bZ):=[ \Pi_{r}(\X), \bGamma ]$. It is well-known that the rank-$r$ projection above can be explicitly computed by the singular value decomposition (SVD). Namely, $\Pi_{r}(\X)= \U\bSigma\V^{T}$, where $\bSigma$ is the $r\times r$ diagonal matrix of the top $r$ singular values of $\X$ and  $\U\in \R^{d_{1}\times r}$, $\V\in \R^{d_{2}\times r}$ are semi-orthonormal matrices (i.e., $\U^{T}\U = \V^{T}\V = \I_{r}$). 
Note that algorithm \eqref{eq:LRPGD_iterate0} resembles the standard \textit{projected gradient descent} (PGD) commonly used in the optimization literature. The algorithm follows a three-step procedure where a gradient descent step is performed, followed by projection onto the convex constraint set $\Param$ and subsequently the rank-$r$ projection. 
It is also worth noting the similarity of \eqref{eq:LRPGD_iterate0} to the `lift-and-project' algorithm in \cite{chu2003structured} for structured low-rank approximation problem, which proceeds by alternatively applying the projections $\Pi_{\Param}$ and $\Pi_{r}$ to a given matrix until convergence.  

In Theorem \ref{thm:CALE_LPGD}, we will show that the iterate $\bZ_{t}$ of the algorithm \eqref{eq:LRPGD_iterate0} converges exponentially to a low-rank approximation of the global minimizer of the objective $f$ over $\Param$, given that the objective $f$ satisfies the following restricted strong convexity (RSC) and restricted smoothness (RSM) properties in Definition \ref{def:RSC}. These properties were first used in  \cite{agarwal2010fast, ravikumar2011high,negahban2011estimation} for the class of matrix estimation problems and have found a number of applications in optimization and machine learning literature \cite{wang2017unified,park2018finding, tong2021accelerating}.

\begin{definition}(Restricted Strong Convexity and Smoothness)\label{def:RSC}
	A function $f:\R^{d_{1}\times d_{2}} \times \R^{d_{3}\times d_{4}} \rightarrow \R$ is \textit{$r$-restricted strongly convex and smooth} with parameters $\mu,L>0$ if for all $\bZ,\bZ'\in \R^{d_{1}\times d_{2}}\times \R^{d_{3}\times d_{4}}$  whose matrix coordinates are of rank $\le r$, 
	\begin{align}\label{eq:def_RSC_RSM}
		\frac{\mu}{2} \lVert \vect(\bZ) - \vect(\bZ') \rVert_{2}^{2} \overset{\textup{(RSC)}}{\le} 	f(\bZ') - f(\bZ) - \langle \nabla f(\bZ),\, \bZ'-\bZ \rangle  \overset{\textup{(RSM)}}{\le} \frac{L}{2} \lVert \vect(\bZ) - \vect(\bZ') \rVert_{2}^{2}.
	\end{align}
\end{definition}

Next, we discuss optimality measures for the CALE problem  \eqref{eq:CALE}. Recall that we want to minimize the objective $f$ subject to two constraints: (1) convex constraint $\Param$ and (2) low-rank constraint. We first consider the following simpler problem without the low-rank constraint:
\begin{align}\label{eq:convex_only}
	\min_{\bZ\in \Param} f(\bZ). 
\end{align}
A first-order optimal point $\bZ^{*}$ for the above problem is called a \textit{stationary point} of $f$ over $\Param$, which is defined as 
\begin{align}
	\langle \nabla f (\bZ^{*}),\, \bZ-\bZ^{*} \rangle \ge 0 \quad \textup{for all $\bZ\in \Param$}. 
\end{align}
An alternative definition of stationary points uses gradient mapping \cite{nesterov2013gradient, beck2017first}, which is particularly well-suited for projected gradient descent type algorithms. Define a map $G:\Param\times (0,\infty)\rightarrow \R$ by 
\begin{align}\label{eq:def_grad_mapping}
	G(\bZ, \tau) := \frac{1}{\tau}(\bZ - \Pi_{\Param}(\bZ - \tau \nabla f(\bZ))).
\end{align}
We call $G$ the \textit{gradient mapping} associated with problem \eqref{eq:convex_only}. In order to motivate the definition, fix $\bZ\in \Param$ and decompose it as 
\begin{align}
	\bZ &=\Pi_{\Param}(\bZ - \tau \nabla f(\bZ))   + \left( \bZ- \Pi_{\Param}(\bZ - \tau \nabla f(\bZ))   \right) \\
	&=\Pi_{\Param}(\bZ - \tau \nabla f(\bZ))  +  \tau G(\bZ, \tau).
\end{align}
Namely, the first term above is a one-step update of a projected gradient descent at $\bZ$ over $\Param$ with stepsize $\tau$, and the second term above is the error term. If $\bZ$ is a stationary point of $f$ over $\Param$, then $-\nabla f(\bZ)$ lies in the normal cone of $\Param$ at $\bZ$, so $\bZ$ is invariant under the projected gradient descent and the error term above is zero. If $\bZ$ is only approximately stationary, then the error above is nonzero. In fact, $ G(\bZ, \tau)=0$ if and only if $\bZ$ is a stationary point of $f$ over $\Param$ (see Theorem 10.7 in \cite{beck2017first}). Therefore, we may use the size of $ G(\bZ,\tau)$ (measured using an appropriate norm) as a measure of first-order optimality of $\bZ$ for the problem \eqref{eq:CALE}. In the special cases when $\Param$ is the whole space or when $\bZ$ is in the interior of $\Param$,  if $\tau$ is sufficiently small (so that $\bZ-\tau\nabla f(\bZ)\in \Param$), then $\lVert G(\bZ, \tau)\rVert_{F} = \lVert \nabla f(\bZ) \rVert_{F}$, which is the standard measure of first-order optimality of $\bZ$ for minimizing the objective $f$. In general, it holds that $\lVert G(\bZ,\tau)\rVert_{F}\le \lVert \nabla f(\bZ) \rVert_{F}$ (see Lemma \ref{lem:gradient_mapping}). 

Now we turn our attention to \eqref{eq:CALE}. An optimal solution for \eqref{eq:convex_only} need not be an optimal solution for \eqref{eq:CALE}, since it may or may not satify the low-rank constraint. Our theoretical convergence guarantee of the LPGD algorithm \eqref{eq:LRPGD_iterate0} for CALE \eqref{eq:CALE} covers these two cases.

\begin{theorem}(Exponential convergence  of LPGD)\label{thm:CALE_LPGD}
	Let $f:\R^{d_{1}\times d_{2}} \times \R^{d_{3}\times d_{4}} \rightarrow \R$ be twice differentiable and $r$-restricted strongly convex and smooth with parameters $\mu$ and $L$, respectively, with $L/\mu<3$.  Let $(\bZ_{t})_{t\ge 0}$ be the iterates generated by algorithm \eqref{eq:LRPGD_iterate0}. Suppose $\Param\subseteq \R^{d_{1}\times d_{2}} \times \R^{d_{3}\times d_{4}}$ is a convex subset and fix a stepsize $\tau\in ( \frac{1}{2\mu}, \frac{3}{2L})$. Then the `contraction constant' $\rho:=2 \max( |1-\tau\mu|,\, |1-\tau L| )<1$ and the followings hold:
	
	\begin{description}
		\item[(i)] (Correctly specified case) Suppose $\bZ^{\star}=[\X^{\star}, \bGamma^{\star}]$ is a stationary point of $f$ over $\Param$ such that $\rank(\X^{\star})\le r$. Then $\bZ^{\star}$ is the unique global minimizer of $\eqref{eq:CALE}$, $\lim_{t\rightarrow\infty} \bZ_{t}= \bZ^{\star}$, and for $t\ge 1$, 
		\begin{align}
			\lVert \bZ_{t} - \bZ^{\star}  \rVert_{F}\le  \rho^{t}  \, \lVert  \bZ_{0} - \bZ^{\star}\rVert_{F}.
		\end{align}
		
		\item[(ii)] (Possibly misspecified case) Let $\bZ^{\star}=[\X^{\star}, \bGamma^{\star}]$ be an arbitrary point in the interior of $\Param$ with $\rank(\X^{\star})\le r$. Then for $t\ge 1$, 
		\begin{align}\label{eq:CALE_LPGD_thm0}
			\lVert \bZ_{t} - \bZ^{\star}  \rVert_{F}  \le  \rho^{t}  \, \lVert  \bZ_{0} - \bZ^{\star}\rVert_{F} + \frac{\tau}{1-\rho}  \left( \sqrt{3r} \lVert  \nabla_{\X} f (\bZ^{\star})  \rVert_{2} + \lVert   \nabla_{\bGamma}  f (\bZ^{\star})\rVert_{F} \right).
		\end{align}
		In general, if $\bZ^{\star}$ is an arbitrary point of $\Param$ with $\rank(\X^{\star})\le r$, then denoting the gradient mapping 
		$[\Delta \X^{\star}, \Delta \bGamma^{\star}]:=\frac{1}{\tau}(\bZ^{\star} - \Pi_{\Param}(\bZ^{\star} - \tau \nabla f(\bZ^{\star})))$ at $\bZ^{\star}$, then for $t\ge 1$, 
		\begin{align}\label{eq:CALE_LPGD_thm}
			\lVert \bZ_{t} - \bZ^{\star}  \rVert_{F}  \le  \rho^{t}  \, \lVert  \bZ_{0} - \bZ^{\star}\rVert_{F} + \frac{2\tau}{1-\rho}  \left( \sqrt{3r} \lVert  \Delta \X^{\star}  \rVert_{2} + \lVert   \Delta \bGamma^{\star} \rVert_{F} \right).
		\end{align}
		
	\end{description}
\end{theorem}

Theorem \ref{thm:CALE_LPGD} \textbf{(i)} asserts that the LPGD algorithm \eqref{eq:LRPGD_iterate0} converges at a linear rate to the unique global minimizer $\bZ^{\star}$, provided that there exists a stationary point $\bZ^{\star}$ of $f$ over the convex constraint set $\Param$ with the first matrix factor $\X$ having rank at most $r$. This assumption holds in the standard statistical estimation setting, where one seeks to estimate a `ground-truth' parameter $\bZ^{\star}$ with a low-rank matrix factor from noisy observations. In this case, the objective $f$ represents the empirical error. Hence in this case, it is reasonable to assume that the gradient $\nabla f(\bZ^{\star})$ (in general, the gradient mapping $G(\bZ, \tau)$) is small or at least $\bZ^{\star}$ is near-stationary. In fact, Wang et al. \cite[Condition 5.7]{wang2017unified} makes such an assumption. 

In contrast, Theorem \ref{thm:CALE_LPGD} does not require such an assumption of near-optimality of the parameter $\bZ^{\star}$ to be estimated. In practical situations, the rank of the ground-truth parameter is often unknown, and one attempts to explain observed data by using a low-rank model, in which case the assumed rank $r$ could be much lower than the true rank. For such generic situations, let $\mathbf{Z}^{\star}$ be an admissible parameter such that the second term in \eqref{eq:CALE_LPGD_thm} is minimized. Then Theorem \ref{thm:CALE_LPGD} \textbf{(ii)} shows that the LPGD algorithm \eqref{eq:LRPGD_iterate0} converges linearly to such $\bZ^{\star}$ up to a `model misspecification error', the minimum value of the second term in \eqref{eq:CALE_LPGD_thm}. In the proof of statistical estimation guarantees of SMF stated in Theorems \ref{thm:SMF_LPGD} and \ref{thm:SMF_LPGD_STAT}, we show that such a model misspecification error is small with high probability.

The general framework of proof in Theorem \ref{thm:CALE_LPGD} shares similarities with the standard argument used to demonstrate the exponential convergence of projected gradient descent with a fixed step size for constrained strongly convex problems (as shown in Theorem 10.29 in \cite{beck2017first}). However, a key technical challenge arises due to the absence of non-expansiveness (i.e., 1-Lipschitzness) in our case. This challenge stems from the fact that the constraint set of low-rank matrices is not convex when minimizing a strongly convex objective with a rank-constrained matrix parameter. Consequently, we cannot rely on the non-expansiveness of the convex projection operator, especially considering that the rank-$r$ projection $\Pi_{r}$ obtained via truncated SVD is not guaranteed to be non-expansive. 


To address this issue, we employ a strategy that involves comparing the iterates $\bZ_{t}$ obtained from \eqref{eq:LRPGD_iterate0} with auxiliary iterates $\hat{\bZ}_{t}$. These auxiliary iterates are derived using a carefully designed linear projection (see Lemma \ref{lem:rank_r_lin_appx}) that incorporates non-expansiveness. Then we can establish that the original rank-$r$ projection is essentially $2$-Lipschitz. So if the distance between the auxiliary iterate $\hat{\bZ}_{t}$ and the global minimizer contracts with a ratio $<1/2$, then the distance between the actual iterate $\bZ_{t}$ and the global minimizer contracts with a ratio $<1$. This contraction property ensures exponential convergence of the distance between $\bZ_{t}$ and the global minimizer in Theorem \ref{thm:CALE_LPGD}.

\begin{proof}[\textbf{Proof of Theorem} \ref{thm:CALE_LPGD}]
	
	We first derive \textbf{(i)} assuming \textbf{(ii)}. Suppose $\bZ^{\star}=[\X^{\star}, \bGamma^{\star}]$ is a stationary point of $f$ over $\Param$ such that $\rank(\X^{\star})\le r$. Let $\bZ=[\X, \bGamma]$ be arbitrary in $\Param$ with $\rank(\X)\le r$. By stationarity of $\bZ^{\star}$, we have $\langle \nabla f(\bZ^{\star}),\, \bZ-\bZ^{\star} \rangle \ge 0$, so by RSC \eqref{eq:def_RSC_RSM}, 
	\begin{align}
		\frac{\mu}{2} \lVert \vect(\bZ) - \vect(\bZ^{\star}) \rVert^{2} \le f(\bZ) -  f(\bZ^{\star}). 
	\end{align}
	Hence $f(\bZ)\ge f(\bZ^{\star})$.  Thus $\bZ^{\star}$ is the unique global minimizer of $\eqref{eq:CALE}$. Also, since $\bZ^{\star}$ is a stationary point of $f$ over $\Param$, the gradient mapping $\frac{1}{\tau}(\bZ^{\star} - \Pi_{\Param}(\bZ^{\star} - \tau \nabla f(\bZ^{\star})))$ is zero. Thus the rest of \textbf{(i)} follows from \textbf{(ii)}. 
	
	Next, we prove \textbf{(ii)}. Let $\bZ^{\star}=[\X^{\star}, \bgamma^{\star}]\in \Param$ be arbitrary with $\rank(\X^{\star})\le r$. Fix an iteration counter $t\ge 1$. Our proof consists of several steps.

	
	\textbf{Step 1: Constructing a suitable linear projection}
	
	
	Let $\X^{\star} = \U^{\star} \bSigma^{\star} (\V^{\star})^{T}$ denote the SVD of $\X^{\star}$.  For each iteration $t$, denote $\bZ_{t} = [\X_{t}, \bgamma_{t}]$ and let $\X_{t} = \U_{t} \bSigma_{t} \V_{t}^{T}$ denote the SVD of $\X_{t}$. Since $\X_{t}$ and $\X^{\star}$ have rank at most $r$, all of both $\U^{\star}$, $\U_{t}$, $\V^{\star}$, and $\V_{t}$ have at most $r$ columns. Define a matrix $\U_{3r}$ so that its columns form an orthonormal basis for the subspace spanned by the columns of $[\U^{\star}, \U_{t-1}, \U_{t}]$. Then $\U_{3r}$ has at most $3r$ columns. Similarly, let  $\V_{3r}$ be a matrix so that its columns form an orthonormal basis for the subspace spanned by the columns of $[\V^{\star}, \V_{t-1}, \V_{t}]$. Then $\V_{3r}$ has at most $3r$ columns. Now, define the subspace 
	\begin{align}\label{eq:mathcal_A_construction}
		\mathcal{A} := \left\{ \Delta\in \R^{d_{1}\times d_{2}}\,|\, \textup{span}(\Delta^{T}) \subseteq \textup{span}(\V_{3r}),\, \textup{span}(\Delta) \subseteq \textup{span}(\U_{3r})  \right\}.
	\end{align}
	Note that $\mathcal{A}$ is a convex subset of $\R^{d_{1}\times d_{2}}$.  Also note that, by definition, $\X^{\star}, \X_{t}, \X_{t-1}\in \mathcal{A}$. Let $\Pi_{\mathcal{A}}$ denote the linear projection operator of $\R^{d_{1}\times d_{2}}$ onto $\mathcal{A}$. 
	
	
	\textbf{Step 2: Constructing auxiliary iterates $\hat{\bZ}_{t}$} 
	
	Let $\mathcal{A}$ denote the linear subspace of $\R^{d_{1}\times d_{2}}$ in \eqref{eq:mathcal_A_construction}. Let 
	\begin{align}
		\Pi':=\Pi_{\mathcal{A}\times \R^{d_{3}\times d_{4}}}
	\end{align}
	denote the projection operator of $\R^{d_{1}\times d_{2}}\times \R^{d_{3}\times d_{4}}$ onto $\mathcal{A}\times \R^{d_{3}\times d_{4}}$.
	Define the following auxiliary iterates
	\begin{align}
		\hat{\bZ}_{t}=[\hat{\X}_{t}, \bGamma_{t}]:=\Pi'\left(  \Pi_{\Param} \left( \bZ_{t-1} -  \tau \nabla f(\bZ_{t-1}) \right) \right).
	\end{align}
	By Lemma \ref{lem:rank_r_lin_appx} and the choice of $\mathcal{A}$, we have 
	\begin{align}\label{eq:mathcal_A_subspace_cond}
		\X_{t} = \Pi_{r}(\hat{\X}_{t}) \in \argmin_{\X, \rank(\X)\le r} \lVert \hat{\X}_{t} - \X  \rVert_{F} \quad \textup{and} \quad  \bZ_{t}, \bZ_{t-1}, \bZ^{\star}\in \mathcal{A} \times \R^{d_{3}\times d_{4}}.
	\end{align}
	It follows that 
	\begin{align}
		\lVert \bZ_{t} - \bZ^{\star}  \rVert_{F}  &\le 	\lVert \bZ_{t} - \hat{\bZ}_{t}  \rVert_{F} + 	\lVert \hat{\bZ}_{t} - \bZ^{\star}  \rVert_{F}  \\
		&= 	\lVert \X_{t} - \hat{\X}_{t}  \rVert_{F} + 	\lVert \hat{\bZ}_{t} - \bZ^{\star}  \rVert_{F}  \\
		&\le 	\lVert\X^{\star} -  \hat{\X}_{t}  \rVert_{F}  + 	\lVert \hat{\bZ}_{t} - \bZ^{\star}  \rVert_{F}   \le 2 	\lVert \hat{\bZ}_{t} - \bZ^{\star}  \rVert_{F}. \label{eq:z_hat_twice_bd}
	\end{align}
	Hence if we can show $\lVert \hat{\bZ}_{t} - \bZ^{\star}  \rVert_{F}$ is small, then $\lVert \bZ_{t} - \bZ^{\star}  \rVert_{F}$ is also small. 
	

	\textbf{Step 3. Showing $\lVert \hat{\bZ}_{t} - \bZ^{\star}  \rVert_{F}$ is small}

	Denote the gradient mapping $\Delta \bZ^{\star}:=\frac{1}{\tau}(\bZ^{\star} - \Pi_{\Param}( \bZ^{\star} - \tau \nabla f(\bZ^{\star})))$ (Recall that this equals zero if $\bZ^{\star}$ is a stationary point of $f$ over $\Param$, but we do not make such an assumption in this proof). We claim that 
	\begin{align}\label{eq:pf_LPGD_pf1}
		\lVert \hat{\bZ}_{t} - \bZ^{\star}  \rVert_{F}  \le  \eta \, \lVert  \bZ_{t-1} - \bZ^{\star}\rVert_{F} + \lVert \Pi'\left( \tau \Delta \bZ^{\star}\right)  \rVert_{F},
	\end{align}
	where $\eta:=\max( |1-\tau L| ,\, |1-\tau \mu| )$. 
	
	Below we show \eqref{eq:pf_LPGD_pf1}. Using $\bZ^{\star}\in \mathcal{A}\times \R^{d_{3}\times d_{4}}$ and linearity of the linear projection $\Pi'$, write 
	\begin{align}
		\bZ^{\star} &= \Pi'(\bZ^{\star}) \\ 
		&=\Pi' \left( \Pi_{\Param}(\bZ^{\star} - \tau \nabla f(\bZ^{\star})) \right)  + \Pi' \left( \bZ^{\star}-  \Pi_{\Param}(\bZ^{\star} - \tau \nabla f(\bZ^{\star}))  \right) \\
		&=\Pi' \left( \Pi_{\Param}(\bZ^{\star} - \tau \nabla f(\bZ^{\star})) \right) +  \Pi' \left( \tau \Delta \bZ^{\star} \right).\label{eq:pf_LPGD_pf11}
	\end{align}
	Using the non-expansiveness of $\Pi'$ and $\Pi_{\Param}$ and linearity $\Pi'$, 
	\begin{align}\label{eq:pf_LPGD_pf111}
		&\lVert \hat{\bZ}_{t} - \bZ^{\star} \rVert_{F}  \\
		&\qquad = \left\lVert  \begin{matrix} \Pi'\left(  \Pi_{\Param} \left( \bZ_{t-1} -  \tau \nabla f(\bZ_{t-1}) \right) \right) 
			- \Pi'\left(  \Pi_{\Param} \left( \bZ^{\star} 
			-  \tau \nabla f(\bZ^{\star}) \right) \right) - \Pi'\left( \tau \Delta \bZ^{\star} \right) \end{matrix} \right\rVert_{F}  \\
		&\qquad \le  \left\lVert  \bZ_{t-1} -  \tau \nabla f(\bZ_{t-1})  - \bZ^{\star} +  \tau \nabla f(\bZ^{\star})  \right\rVert_{F}  + \lVert \Pi'\left( \tau \Delta \bZ^{\star}  \right)  \rVert_{F}.
	\end{align}

	Hence in order to derive \eqref{eq:pf_LPGD_pf1}, it is enough to show that 
	\begin{align}\label{eq:pf_LPGD_pf2}
		\lVert \bZ_{t-1} -  \tau \nabla f(\bZ_{t-1})  - \bZ^{\star} +  \tau \nabla f(\bZ^{\star}) \rVert_{F} \le \eta  \lVert \bZ_{t-1} - \bZ^{\star}\rVert_{F}.
	\end{align}
	The above follows from the fact that $\bZ_{t-1}$ and $\bZ^{\star}$ have rank $\le r$ and the restricted strong convexity and smoothness properties (Definition \ref{def:RSC}). Indeed,  since $\nabla^2 f$ is continuous, 
	\begin{align}
		&\bZ_{t-1} -  \tau \nabla f(\bZ_{t-1})  - \bZ^{\star} +  \tau \nabla f(\bZ^{\star}) \\ &\qquad = (\bZ_{t-1}-\bZ^{\star}) - \tau ( \nabla f(\bZ_{t-1}) -  \nabla f(\bZ^{\star})) \\
		&\qquad = \int_{0}^{1} \left(  \I-\tau \nabla^{2} f(\bZ^{\star} + s(\bZ_{t-1}-\bZ^{\star})) \right)(\bZ_{t-1}-\bZ^{\star})\,ds.
	\end{align}
	From the above with the inequality $\lVert \A \B \rVert_{F} \le \lVert \A \rVert_{2} \lVert \B \rVert_{F}$, 
	\begin{align}
		&\lVert \bZ_{t-1} -  \tau \nabla f(\bZ_{t-1})  - \bZ^{\star} +  \tau \nabla f(\bZ^{\star}) \rVert_{F} \\
		&\qquad  \le \sup_{\tilde{\bZ}=[\bZ_{1},\bZ_{2}]: \, \rank(\bZ_{1})\le r}   \lVert  \I-\tau \nabla^{2}f( \tilde{\bZ})\rVert_{2}  \, \lVert \bZ_{t-1} - \bZ^{\star} \rVert_{F}.
	\end{align}
	Since the eigenvalues of $\nabla^{2} f(\bZ_{t-1})$ are contained in $[\mu,L]$, the eigenvalues of $\I-\tau \nabla^{2}f(\bZ_{t-1}) $ are between $1-\tau L$ and $1-\tau \mu$.  Hence the right hand side above is at most 
	\begin{align}\label{eq:pf_LPGD_pf211}
		\eta \,\lVert \bZ_{t-1} - \bZ^{\star} \rVert_{F},
	\end{align}
	verifying \eqref{eq:pf_LPGD_pf2}. This shows \eqref{eq:pf_LPGD_pf1}. 
	
	\textbf{Step 4: Bounding the error term}
	
	From \eqref{eq:z_hat_twice_bd} and  \eqref{eq:pf_LPGD_pf1}, we deduce
	\begin{align}
		\lVert \bZ_{t} - \bZ^{\star}  \rVert_{F}  \le  2\eta \, \lVert  \bZ_{t-1} - \bZ^{\star}\rVert_{F}+ 2\lVert \Pi'\left(\tau\Delta \bZ^{\star}\right)  \rVert_{F}.
	\end{align}
	Note that  $\Pi'(\Delta \X^{\star}, \Delta \bgamma^{\star}) = [\Pi_{\mathcal{A}}(\Delta \X^{\star}) , \Delta \bgamma^{\star} ]$ and  $\rank(\mathcal{A})\le 3r$. Thus by triangle inequality, 
	\begin{align}\label{eq:LPGD_last_gap_bound}
		\lVert \Pi'\left( \Delta \X^{\star}, \Delta \bgamma^{\star}  \right)  \rVert_{F} &\le \lVert \Pi'( \Delta \X^{\star})  \rVert_{F} + \lVert  \Delta \bgamma^{\star} \rVert_{F} \\
		&\le \sqrt{ 3r} \lVert  \Delta \X^{\star}  \rVert_{2} +  \lVert \Delta \bgamma^{\star} \rVert_{F}.
	\end{align}

	Also note that $0\le \eta<1/2$ if and only if $\tau\in (\frac{1}{2\mu}, \frac{3}{2L})$, and this interval is non-empty if and only if $L/\mu<3$. Hence for such choice of $\tau$, $0<2\eta<1$,  so  
	by a recursive application of the above inequality, we obtain 
	\begin{align}\label{eq:linear_conv_ineq_pf1}
		\lVert \bZ_{t} - \bZ^{\star}  \rVert_{F}  \le  (2\eta)^{t}\, \lVert  \bZ_{0} - \bZ^{\star}\rVert_{F} +  \frac{2\tau }{1-2\eta} \left( \sqrt{ 3r} \lVert  \Delta \X^{\star}  \rVert_{2} +  \lVert \Delta \bgamma^{\star} \rVert_{F} \right).
	\end{align}
	
	This completes the proof of \textbf{(ii)}.
\end{proof}

The following lemma is inspired by the proof of Thm. 5.9 in \cite{wang2017unified}. 
\begin{lemma}(Linear projection factoring through rank-$r$ projection)\label{lem:rank_r_lin_appx}
	Fix $\Y\in \R^{d_{1}\times d_{2}}$, $R\ge r \in \mathbb{N}$, and denote $\X=\Pi_{r}(\Y)$ and $\hat{\X} = \Pi_{\mathcal{A}}(
	\Y)$, where $\mathcal{A}\subseteq \R^{d_{1}\times d_{2}}$ is a linear subspace. Let $\X=\U\bSigma \V^{T}$ denote the SVD of $\X$. Suppose there exists orthonormal matrices $\overline{\U}\in \R^{d_{1}\times R}$ and $\overline{\V} \in \R^{d_{2}\times R}$ such that 
	\begin{align}
		&\mathcal{A} = \left\{ \A \in \R^{d_{1}\times d_{2}}\,\big|\, \textup{col}(\A^{T}) \subseteq \textup{col}(\overline{\V}),\, \textup{col}(\A) \subseteq \textup{col}(\overline{\U})  \right\}, \\
		&\quad \textup{col}(\U)\subseteq \textup{col}(\overline{\U}), \quad \textup{col}(\V)\subseteq \textup{col}(\overline{\V}).
	\end{align}
	Then $\X=\Pi_{r}(\hat{\X})$. 
\end{lemma}

\begin{proof}
	Write $\Y-\X=\dot{\U} \dot{\bSigma} \dot{\V}^{T}$  for its SVD. Let $d:=\rank(\Y)$ and let $\sigma_{1}\ge \dots \ge \sigma_{d}>0$ denote the nonzero singular values of $\Y$. Since $\X=\Pi_{r}(\Y) = \U\bSigma\V^{T}$ and $\Y=\U\bSigma\V^{T} + \dot{\U} \dot{\bSigma} \dot{\V}^{T}$, we must have that $\bSigma$ consists of the top $r$ singular values of $\Y$ and the rest of $d-r$ singular values are contained in $\dot{\bSigma}$. Furthermore, $\textup{col}(\U) \perp \textup{col}(\dot{\U})$.
	
	Now, since $\X\in \mathcal{A}$ and $\Pi_{\mathcal{A}}$ is linear, we get 
	\begin{align}\label{eq:3r_r_approx1_lem}
		\hat{\X} = \Pi_{\mathcal{A}}( \X + (\Y -\X) ) = \U \bSigma \V^{T} + \Pi_{\mathcal{A}}(\dot{\U} \dot{\bSigma} \dot{\V}^{T}). 
	\end{align}
	Let $\bZ:= \Pi_{\mathcal{A}}(\dot{\U} \dot{\bSigma} \dot{\V}^{T})$ and write its SVD as $\bZ = \widetilde{\U}\widetilde{\bSigma}\widetilde{\V}^{T}$. Then note that $(\U^{T}\overline{\U} \, \overline{\U}^{T})^{T}=\overline{\U}\,\overline{\U}^{T}\U=\U$ since $\overline{\U}\,\overline{\U}^{T}:\R^{d_{1}}\rightarrow \R^{d_{1}}$ is the orthogonal projection onto $\textup{col}(\overline{\U})\supseteq \textup{col}(\U)$. Hence $\U^{T}\overline{\U}\,\overline{\U}^{T} = \U^{T}$. Also note that, by the definition of $\mathcal{A}$, for each $\mathbf{B}\in \R^{d_{1}\times d_{2}}$, 
	\begin{align}
		\Pi_{\mathcal{A}}(\mathbf{B}) = \overline{\U}\overline{\U}^{T} \mathbf{B} \overline{\V}^{T}\overline{\V}. 
	\end{align}
	Hence, noting that $\textup{col}(\U)\perp \textup{col}(\dot{\U})$, we get 
	\begin{align}
		\U^{T} \bZ &=  	\U^{T} \left( \overline{\U}\, \overline{\U}^{T} \dot{\U} \dot{\bSigma} \dot{\V}^{T} 	\overline{\V}^{T} \overline{\V}\right) \\
		&= \left(  \U^{T} \overline{\U}\, \overline{\U}^{T} \right)\dot{\U} \dot{\bSigma} \dot{\V}^{T} 	\overline{\V}^{T} \overline{\V} \\
		&= \left( \U^{T} \dot{\U} \right) \dot{\bSigma} \dot{\V}^{T} 	\overline{\V}^{T} \overline{\V} = O.
	\end{align}
	It follows that $\U^{T}\widetilde{\U}=O$, since $	\U^{T}\widetilde{\U} = (\U^{T} \bZ) \widetilde{\V} (\widetilde{\bSigma})^{-1} = O$. Therefore, rewriting \eqref{eq:3r_r_approx1_lem} gives the SVD of $\hat{\X}$ as 
	\begin{align}
		\hat{\X} = \begin{bmatrix} 
			\U & \widetilde{\U}
		\end{bmatrix}
		\begin{bmatrix} 
			\bSigma & O\\
			O & \widetilde{\bSigma} 
		\end{bmatrix}
		\begin{bmatrix} 
			\V \\ \widetilde{\V}
		\end{bmatrix}
		.
	\end{align}
	Furthermore, $\lVert  \Pi_{\mathcal{A}}(\dot{\U} \dot{\bSigma} \dot{\V}^{T}) \rVert_{2} \le \lVert\dot{\bSigma} \rVert_{2} =\sigma_{r+1}$, so $\bSigma$ consists of the top $r$ singular values of $\hat{\X}$. It follows that $\X=\U \bSigma \V^{T}$ is the best rank-$r$ approximation of $\hat{\X}$, as desired. 
\end{proof}

\begin{remark}\label{rmk:pf_thm_LPGD}
	Note that in \eqref{eq:LPGD_last_gap_bound}, we could have used the following crude bound 
	\begin{align}\label{eq:LPGD_last_gap_bound2}
		\left\lVert \Pi'\left( \Delta \X^{\star}, \Delta \bgamma^{\star}  \right)  \right\rVert_{F} \le 	\left\lVert \left[  \Delta \X^{\star}, \Delta \bgamma^{\star}  \right]  \right\rVert_{F} &\le 
		\lVert \Delta \X^{\star} \rVert_{F} +  \lVert \Delta \bgamma^{\star} \rVert_{F} \\
		&\le \sqrt{\rank(\Delta \X^{\star})}	\lVert \Delta \X^{\star} \rVert_{2} +  \lVert \Delta \bgamma^{\star} \rVert_{F},
	\end{align}
	which is also the bound we would have obtained if we chose the trivial linear subspace $\mathcal{A}=\R^{d_{1}\times d_{2}}$ in the proof of Theorem \ref{thm:CALE_LPGD} above. While we know $\rank(\X^{\star})\le r $, we do not have an a priori bound on $ \rank(\Delta \X^{\star})$, which could be much larger than $\sqrt{3r}$. A smarter choice of the subspace $\mathcal{A}$ as we used in the proof of Theorem \ref{thm:CALE_LPGD} ensures that we only need the factor $\sqrt{3r}$ in place of the unknown factor $\sqrt{\rank(\Delta \X^{\star})}$ as in \eqref{eq:LPGD_last_gap_bound}. 
\end{remark}

\begin{remark}
	Suppose $f$ is not only rank-restricted smooth, but also $L'$-smooth on $\Param$ for some $L'>0$. Then we have 
	\begin{align}\label{eq:PSGD_linear_conv2}
		f\left( \bZ_{t}
		\right)  -  f(\bZ^{\star}) \le \left( \lVert \nabla f(\bZ^{\star}) \rVert + L \rho^{t} \right) \rho^{t} \lVert \bZ_{0}-\bZ^{\star}\rVert_{F}
	\end{align}
	for $t\ge 1$. Indeed, note that 
	\begin{align}
		\left| f(\bZ_{n}) - f(\bZ^{\star}) \right| & = \left|  \int_{0}^{1} \left\langle \nabla f \left(\bZ^{\star} + s(\bZ_{n} - \bZ^{\star}) \right),\, \bZ_{n}-\bZ^{\star} \right\rangle \,ds \right| \\
		&\le   \int_{0}^{1}\left\lVert  \nabla f \left(\bZ^{\star} + s(\bZ_{n} - \bZ^{\star}) \right) \right\rVert \lVert \bZ_{n}-\bZ^{\star} \rVert \,ds \\ 
		&\le   \int_{0}^{1} \left( \left\lVert  \nabla f (\bZ^{\star}) \rVert + s L'\lVert \bZ_{n}-\bZ^{\star} \right\rVert  \right) \lVert \bZ_{n}-\bZ^{\star} \rVert \,ds \\ 
		&\le   \left( \left\lVert  \nabla f (\bZ^{\star}) \rVert +  L'\lVert \bZ_{n}-\bZ^{\star} \right\rVert  \right) \lVert \bZ_{n}-\bZ^{\star} \rVert.
	\end{align}
	Then \eqref{eq:PSGD_linear_conv2} follows from Theorem \ref{thm:CALE_LPGD} \textbf{(ii)}. 
\end{remark}

\begin{remark}
	A similar approach  as in our proof of Theorem \ref{thm:CALE_LPGD} was used in \cite{wang2017unified} for analyzing a similar problem without auxiliary features and under a stronger assumption that the gradient $\nabla f(\bZ^{\star})$ is small. Our analysis is for a more general setting but is a bit simpler and gives a weaker requirement $L/\mu<3$ for the well-conditioning of the objective $f$ instead of $L/\mu<4/3$ in \cite{wang2017unified}.
\end{remark}

\begin{remark}\label{rmk:projection}
	We give some salient remarks on the use of projections $\Pi_{\Param}$ and $\Pi_{r}$ in our LPGD algorithm \eqref{eq:LRPGD_iterate0}. 
	First, in principle, one could alternate between the two projections $\Pi_{r}$ and $\Pi_{\Theta}$ at every iteration after a gradient descent step until convergence, similarly to the alternating projection in \cite{chu2003structured}. However, this would make each iteration of the algorithm prohibitively expensive as this requries to perform rank-$r$ SVD until convergence at \textit{every iteration}. The problem in \cite{chu2003structured} is much simpler than ours as the objective function is simply the Frobenius norm between the target and the estimated low-rank constrained matrix. 
	
	Second, what happens if we switch the order of two projections $\Pi_{r}$ and $\Pi_{\Theta}$ in \eqref{eq:LRPGD_iterate0}? Our proposed LPGD algorithm performs the convex projection $\Pi_{\Theta}$ first and then applies the low-rank projection $\Pi_{r}$. The key inequality we derive in the proof of Thm. C.2 is \eqref{eq:pf_LPGD_pf1}:
	\begin{align}
		\lVert \mathbf{Z}_{t} - \mathbf{Z}^{\star}  \rVert_{F}  \le  2\eta \, \lVert  \mathbf{Z}_{t-1} - \mathbf{Z}^{\star}\rVert_{F}+ \lVert \Pi_{t}\left(\tau\Delta \mathbf{Z}^{\star}\right)  \rVert_{F},
	\end{align}
	where $\Pi_{t}$ is a linear projection onto a $3r$-dimensional linear subspace that depends on $\mathbf{Z}^{\star}$, $\mathbf{Z}_{t}$, and $\mathbf{Z}_{t-1}$. The last error term above can be bounded above uniformly in $t$ using $\lVert \Pi_{t}( \mathbf{A})\rVert_{F}\le \sqrt{3r} \lVert \mathbf{A} \rVert_{2}$. So we can apply the above inequality recursively to obtain the desired result. 
	
	Now if we consider an alternative algorithm that uses the  low-rank projections $\Pi_{r}$  first and then the convex projection $\Pi_{\Theta}$, then we can derive a corresponding key inequality: 
	\begin{align}\label{eq:projection_rmk}
		\lVert \mathbf{Z}_{t} - \mathbf{Z}^{\star}  \rVert_{F}  \le  2\eta \, \lVert  \mathbf{Z}_{t-1} - \mathbf{Z}^{\star}\rVert_{F}+ \lVert  \tau\Delta^{t} \mathbf{Z}^{\star} \rVert_{F},
	\end{align}
	where $\tau \Delta^{t}\mathbf{Z}^{\star}:=\mathbf{Z}^{\star} - \Pi_{t}(\mathbf{Z}^{\star}-\tau \nabla f(\mathbf{Z}^{\star}))$ denotes the gradient mapping at $\mathbf{Z}^{\star}$ w.r.t. the `virtual' linear constraint that we constructed during the proof to approximate the low-rank constraint. 
	
	To give more detail, it amounts to derive similar inequalities in \eqref{eq:pf_LPGD_pf11}-\eqref{eq:pf_LPGD_pf2} assuming the reverse order of projections. Namely, in place of \eqref{eq:pf_LPGD_pf11}, we use 
	\begin{align}
		\mathbf{Z}^{\star} &= \Pi_{\Theta}(\mathbf{Z}^{\star}) \\ 
		&=\Pi_{\Theta} \left( \Pi'(\mathbf{Z}^{\star} - \tau \nabla f(\mathbf{Z}^{\star})) +   \mathbf{Z}^{\star}-  \Pi'(\mathbf{Z}^{\star} - \tau \nabla f(\mathbf{Z}^{\star}) )\right) .   
	\end{align}
	Note that unlike in \eqref{eq:pf_LPGD_pf11} we cannot distribute the convex projection $\Pi_{\Theta}$ since it is not in general a linear projection as $\Pi_{t}$ is. Then in place of \eqref{eq:pf_LPGD_pf111}, using the non-expansiveness of $\Pi_t$ and $\Pi_{\Theta}$, we get 
	\begin{align}
		&\lVert \hat{\mathbf{Z}}_{t} - \mathbf{Z}^{\star} \rVert_{F}  \\
		&\qquad = \left\lVert  \begin{matrix} \hspace{-4.5cm} \Pi_{\Theta}\left(  \Pi' \left( \mathbf{Z}_{t-1} -  \tau \nabla f(\mathbf{Z}_{t-1}) \right) \right) 
			\\
			\hspace{1cm} - \Pi_{\Theta} \left( \Pi'(\mathbf{Z}^{\star} - \tau \nabla f(\mathbf{Z}^{\star})) +   \mathbf{Z}^{\star}-  \Pi'(\mathbf{Z}^{\star} - \tau \nabla f(\mathbf{Z}^{\star})) \right)   \end{matrix} \right\rVert_{F}  \\
		&\qquad \le \left\lVert  \begin{matrix} \hspace{-4.5cm} \left(  \Pi' \left( \mathbf{Z}_{t-1} -  \tau \nabla f(\mathbf{Z}_{t-1}) \right) \right) 
			\\
			\hspace{1cm} - \left( \Pi'(\mathbf{Z}^{\star} - \tau \nabla f(\mathbf{Z}^{\star})) +   \mathbf{Z}^{\star}-  \Pi'(\mathbf{Z}^{\star} - \tau \nabla f(\mathbf{Z}^{\star}) )\right)   \end{matrix} \right\rVert_{F}  \\
		&\qquad \le  \left\lVert  \mathbf{Z}_{t-1} -  \tau \nabla f(\mathbf{Z}_{t-1})  - \mathbf{Z}^{\star} +  \tau \nabla f(\mathbf{Z}^{\star})  \right\rVert_{F}  + \lVert  \tau \Delta^{t} \mathbf{Z}^{\star}    \rVert_{F}.
	\end{align}
	The first term in the last expression can be bounded by the same argument as in \eqref{eq:pf_LPGD_pf2}-\eqref{eq:pf_LPGD_pf211}. Then by recursively applying the inequality \eqref{eq:projection_rmk}, we can obtain 
	\begin{align}
		\lVert \mathbf{Z}_{t} - \mathbf{Z}^{\star}  \rVert_{F}  \le  (2\eta)^{t}  \, \lVert  \mathbf{Z}_{0} - \mathbf{Z}^{\star}\rVert_{F} + \sum_{k=1}^{t} (2\eta)^{t-k} \lVert  \tau\Delta^{k} \mathbf{Z}^{\star} \rVert_{F}.
	\end{align}
	Hence the rate of convergence we would get is the same as the original algorithm, but the additive error takes a different form. Since the `low-rank gradient mapping' $\Delta^{k}\mathbf{Z}^{\star}$ depends on the iterates $\mathbf{Z}_{k},\mathbf{Z}_{k-1}$, we find it easier to control the gradient mapping with respect to the convex projection that comes out from the analysis of the original algorithm. 
\end{remark}

\section{Proof of Theorems \ref{thm:SMF_LPGD} and \ref{thm:SMF_LPGD_STAT}}\label{sec:thm_proofs}

In this section, we prove the main results for SMF, Theorems \ref{thm:SMF_LPGD} and \ref{thm:SMF_LPGD_STAT}. 
In the main text, we explained that our algorithm for SMF (Alg. \ref{alg:SMF_PGD}) is exactly an LPGD for the reformulated problems \eqref{eq:SMF_feat_CALE0} (for SMF-$\H$) and \eqref{eq:SMF_filt_CALE0} (for SMF-$\W$). Therefore, our proofs of Theorems \ref{thm:SMF_LPGD} and \ref{thm:SMF_LPGD_STAT}  are essentially verifying the well-conditioning hypothesis $L/\mu<3$ of the general result for the LPGD algorithm (Theorem \ref{thm:CALE_LPGD}).

\subsection{Proof of Theorem \ref{thm:SMF_LPGD} and its generalization}  

Theorem \ref{thm:SMF_LPGD} in the main text is a special case of the following more general result, which we prove in this section.

\begin{theorem}(Exponential convergence for SMF)\label{thm:SMF_LPGD_full}
	Let $\bZ_{t}:=[\param_{t}, \bgamma_{t}]$ denote the iterates of Algorithm \ref{alg:SMF_PGD}. Assume \ref{assumption:A2}-\ref{assumption:A4} hold. Let $\mu$ and $L$ be as in \eqref{eq:thm1_condition_numbers}, fix stepsize $\tau\in ( \frac{1}{2\mu}, \frac{3}{2L})$, and let $\rho:=2(1-\tau\mu) \in (0,1)$. Suppose $L/\mu < 3$. 
	\begin{description}
		\item[(i)] (Correctly specified case; Theorem \ref{thm:SMF_LPGD} in the main text)  
		Suppose there exists a stationary point $\bZ^{*}=[\param^{*}, \bgamma^{*}]$ of $F$ over the convex constraint set $\Param$ s.t. $\rank(\param^{*})\le r$. Then $\bZ^{*}$ is the unique global minimizer of $F$ among all $\bZ=[\param,\bgamma]$ with $\rank(\param)\le r$. Moreover, 
		\begin{align}\label{eq:PSGD_linear_conv2_appendix}
			\lVert \bZ_{t} - \bZ^{*}  \rVert_{F}&\le  \rho^{t}  \, \lVert  \bZ_{0} - \bZ^{*}\rVert_{F} \qquad \textup{for $t\ge 1$}. 
		\end{align}
		
		\item[(ii)] (Possibly misspecified case) Let $\bZ^{\star}=[\param^{\star}, \bgamma^{\star}]$ be arbitrary in $\Param$ s.t. $\rank(\param^{\star})\le r$. Denote the gradient mapping at $\bZ^{\star}$ as $[\Delta\param^{\star}, \Delta \bGamma^{\star}] :=\frac{1}{\tau}\left( \bZ^{\star} - \Pi_{\Param}(\bZ^{\star} - \tau \nabla F(\bZ^{\star}) \right)$. Then for $t\ge 1$, 
		\begin{align}\label{eq:LPGD_SMF_filt_thm_appendix}
			\lVert \bZ_{t} - \bZ^{\star}  \rVert_{F} & \le  \rho^{t}  \, \lVert  \bZ_{0} - \bZ^{\star}\rVert_{F}  + \frac{2\tau }{1-\rho}  \left( \sqrt{3r} \lVert \Delta \param^{\star} \rVert_{2} + \lVert \Delta \bgamma^{\star} \rVert_{F} \right). 
		\end{align}
	\end{description}
\end{theorem}

We remark that even in the presence of a nonzero additive error (which bounds the unnormalized estimation error $\lVert \bZ_{\infty} - \bZ^{\star} \rVert_{F}$), our Theorem \ref{thm:SMF_LPGD_STAT} demonstrates that, under natural generative models for SML, this error becomes vanishingly small with high probability with noise variance $\sigma^{2}=O(1/n)$ for SMF-$\W$ and $\sigma^{2}=o(1/\sqrt{n})$ for SMF-$\H$. Roughly speaking, these results indicate that the generative SMF models are nearly correctly specified with high probability. As a result, the algorithm achieves exponential convergence to the correct parameters for the generative SMF model up to a statistical error that vanishes as the sample size tends to infinity.

We begin with some preliminary computations. Let $\a_{s}$ denote the activation corresponding to the $s$th sample (see \eqref{eq:activation}). More precisely,  $\a_{s}=\A^{T}\x_{s}+\bgamma^{T}\x'_{s}$ for the filter-based model with $\A\in \R^{p\times \kappa}$, and $\a_{s}=\A[:,s]+\bgamma^{T}\x'_{s}$ with $\A\in \R^{\kappa\times n}$ for the feature-based model. In both cases, $\B\in \R^{p\times n}$ and $\bgamma\in \R^{q\times \kappa}$. Then the objective function $f$ in \eqref{eq:SMF_main}   can be written as
\begin{align}\label{eq:f_SMF_pf0}
	f(\A, \B, \bgamma)&:=\left( -\sum_{s=1}^{n} \sum_{j=0}^{\kappa} \mathbf{1}(y_{s}=j)  \log g_{j}(  \a_{s}  ) \right) +   \xi \lVert  \X_{\textup{data}}  -\B\rVert_{F}^{2} + \lambda \left( \lVert \A \rVert_{F}^{2}+\lVert \bgamma \rVert_{F}^{2} \right) \\
	&= \sum_{s=1}^{n} \left(  \log \left( 1+\sum_{c=1}^{\kappa} h( \a_{s}[c]  ) \right) -  \sum_{c=1}^{\kappa} \mathbf{1}(y_{s}=c) \log  h( \a_{s}[c]  )  \right)   \\
	&\qquad  + \xi \lVert  \X_{\textup{data}}  -\B\rVert_{F}^{2} + \lambda \left( \lVert \A \rVert_{F}^{2} + \lVert \bgamma \rVert_{F}^{2}  \right),
\end{align}
where $\a_{s}[i]\in \R$ denotes the $i$th component of $\a_{s}\in \R^{\kappa}$. In the proofs we provided below, we compute the Hessian of $f$ above explicitly for the filter- and the feature-based SMF models and use Theorem \ref{thm:CALE_LPGD} to derive the result. 

For each label $y\in \{0,\dots,\kappa\}$ and activation $\a\in \R^{\kappa}$, recall the negative log-likelihood 
\begin{align}\label{eq:LR_gradients_0}
	\ell(y, \a)=\log \sum_{c=1}^{\kappa} h(a_{c})   - \sum_{c=1}^{\kappa}\mathbf{1}_{\{y=c\} }\log h(a_{c})
\end{align}
of observing label $y$ from the probability distribution $\g(\a)$ defined in \eqref{eq:SMF_model}. An easy computation shows 
\begin{align}\label{eq:Hddot_def0}
	\nabla_{\a} \ell(y,\a) = \dot{\h}(y,\a)=(\dot{h}_{1},\dots,\dot{h}_{\kappa})\in \R^{\kappa}, \quad \nabla_{\a}\nabla_{\a^{T}} \ell(y,\a) = \ddot{\H}(y,\a)=(\ddot{h}_{ij})\in \R^{\kappa\times \kappa}, 
\end{align}
where 
\begin{align}\label{eq:Hddot_def}
	\dot{h}_{j} &=\dot{h}_{j}(y,\a) :=\left( \frac{h'(a_{j})}{1+\sum_{c=1}^{\kappa} h(a_{c})} - \mathbf{1}(y=j)\frac{h'(a_{j})}{h(a_{j})} \right)\\
	\ddot{h}_{ij}  &:=  \left(  \frac{ h''(a_{j}) \mathbf{1}(i=j)  }{  1+\sum_{c=1}^{\kappa} h(a_{c})  } - \frac{ h'(a_{i}) h'(a_{j})   }{  \left( 1+\sum_{c=1}^{\kappa} h(a_{c}) \right)^{2} }  \right)  - \mathbf{1}(y=i=j)  \left( \frac{h''(a_{j})}{h(a_{j})} - \frac{\left( h'(a_{j}) \right)^{2} }{\left( h(a_{j}) \right)^{2}}   \right).
\end{align}

\begin{proof}[\textbf{Proof of Theorem} \ref{thm:SMF_LPGD_full} for SMF-$\W$]
	Let $f=F$ denote the loss function for the filter-based SMF model in \eqref{eq:SMF_filt_CALE0}. Fix $\bZ_{1},\bZ_{2}\in \Param\subseteq \R^{d_{1}\times d_{2}}\times \R^{d_{3}\times d_{4}} $. Since the constraint set $\Param$ is convex (see Algorithm \ref{alg:SMF_PGD}), $t\bZ_{1} + (1-t) \bZ_{2}\in \Param$ for all $t\in [0,1]$. Then by the mean value theorem, there exists $t^{*}\in [0,1]$ such that for $\bZ^{*}=t^{*}\bZ_{1} + (1-t^{*})\bZ_{2}$,
	\begin{align}
		&	f(\bZ_{2}) - f(\bZ_{1}) - \langle \nabla f(\bZ_{1}), \, \bZ_{2}-\bZ_{1} \rangle \\
		&\qquad = \frac{1}{2}\left( \vect(\bZ_{2}) - \vect(\bZ_{1})  \right)^{T} \nabla_{\vect(\bZ)}\nabla_{\vect(\bZ)^{T}} f( \bZ^{*} )  \left( \vect(\bZ_{2}) - \vect(\bZ_{1})  \right).
	\end{align}
	Hence, according to Theorem \ref{thm:CALE_LPGD}, it suffices to verify that 
	for some $\mu,L>0$ such that $L/\mu<3$ and 
	\begin{align}\label{eq:RSC_thm_pf}
		\mu \I \preceq \nabla_{\vect(\bZ)}\nabla_{\vect(\bZ)^{T}} f( \bZ^{*} ) \preceq L \I
	\end{align}
	for all $\bZ^{*}=[\X,\bgamma]$ with $\rank(\X^{*})\le r$.

	
	To this end, let  $\a_{s}=\A^{T}\x_{s}+\bgamma^{T}\x'_{s}$ for the filter-based model we consider here. We discussed that the objective function $f$ in \eqref{eq:SMF_filt_CALE0} can be written as \eqref{eq:f_SMF_pf0}. Denote 
	\begin{align}\label{eq:thm_SMF_filt_as_def}
		\a_{s}= \A^{T}\x_{s}+\bgamma^{T}\x_{s}' =: \bigg[ \left\langle \underbrace{\begin{bmatrix}
				\A[:,j] \\ \bgamma[:,j] \end{bmatrix}}_{=:\u_{j}} ,\, \underbrace{\begin{bmatrix} 
				\x_{s} \\ \x_{s}' \end{bmatrix}}_{=:\bphi_{s}}  \right\rangle; \,\, j=1,\dots,\kappa \bigg]^{T}\in \R^{\kappa},
	\end{align}
	where we have introduced the notations $\u_{j}\in \R^{(p+q)\times 1}$ for $j=1,\dots,\kappa$  and $\bphi_{s}\in \R^{(p+q) \times 1 }$ for $s=1,\dots, n$. Denote $\U:=[\u_{1},\dots,\u_{\kappa}]=[\A\parallel \bgamma]\in \R^{(p+q)\times \kappa}$, which is a matrix parameter that vertically concatenates $\A$ and $\bgamma$. Also denote $\bPhi=(\bphi_{1},\dots,\bphi_{n})\in \R^{(p+q)\times n}$, which is the feature matrix of $n$ observations. Writing 
	\begin{align}
		f(\U,\B)=\sum_{s=1}^{n}\ell(y_{s}, \U^{T}\bphi_{s}) + \xi \lVert \X_{\textup{data}}-\B \rVert_{F}^{2} + \lambda \lVert \U \rVert_{F}^{2} 
	\end{align}
	and using \eqref{eq:LR_gradients_0}, we can compute the gradient and the Hessian of $f$ above as follows:
	\begin{align}\label{eq:SMF_filt_gradients}
		&\nabla_{\vect(\U)}  f(\U,\B) = \left(  \sum_{s=1}^{n} \dot{\h}(y_{s},\U^{T}\bphi_{s}) \otimes \bphi_{s} \right) + 2\lambda \vect(\U), \\
		&\nabla_{\B}  f(\U,\B) = 2\xi (\B-\X_{\textup{data}}), \\ 	
		&\nabla_{\vect(\U)}\nabla_{\vect(\U)^{T}}  f(\U,\B) =  \left( \sum_{s=1}^{n} \ddot{\H}(y_{s},\U^{T}\bphi_{s}) \otimes \bphi_{s}\bphi_{s}^{T}\right) + 2\lambda \I_{(p+q)\kappa}, \\
		&\nabla_{\vect(\B)}\nabla_{\vect(\B)^{T}}  f(\U,\B) = 2\xi \I_{pn}, \qquad \nabla_{\vect(\B)}\nabla_{\vect(\U)^{T}}  f(\U,\B) = O,
	\end{align}
	where $\otimes$ above denotes the Kronecker product and the functions $\dot{\h}$ and $\ddot{\H}$ are defined in  \eqref{eq:Hddot_def0}.
	
	Recall that the eigenvalues of $\A\otimes \B$, where $\A$ and $\B$ are two square matrices, are given by $\lambda_{i}\mu_{j}$, where $\lambda_{i}$ and $\mu_{j}$ run over all eigenvalues of $\A$ and $\B$, respectively. Hence denoting $\H_{\U}:=\sum_{s=1}^{n} \ddot{\H}(y_{s},\U^{T}\bphi_{s},) \otimes \bphi_{s}\bphi_{s}^{T}$ and using \ref{assumption:A2}-\ref{assumption:A3}, we can deduce 
	\begin{align}\label{eq:MNL_evals_bounds1}
		\lambda_{\min}(\H_{\U}) &\ge n \lambda_{\min}\left( n^{-1} \bPhi \bPhi^{T} \right) 	\min_{1\le s \le n,\, \U} \lambda_{\min}\left( \ddot{\H}(y_{s},\bphi_{s}, \U) \right) \ge n  \delta^{-}\alpha^{-} \ge n \mu^{*}>0, \\
		\lambda_{\max}(\H_{\U}) &\le n \lambda_{\max}\left( n^{-1}\bPhi \bPhi^{T} \right) 	\max_{1\le s \le n,\, \U} \lambda_{\max}\left( \ddot{\H}(y_{s},\bphi_{s}, \U) \right) \le n \delta^{+}\alpha^{+}\le n L^{*}.
	\end{align}
	It follows that the eigenvalues of the Hessian $\H_{\textup{filt}}$ of the loss function $f$ satisfy
	\begin{align}
		\lambda_{\min}(\H_{\textup{filt}}) & \ge \min( 2\lambda+n \mu^{*}, \, 2\xi), \\
		\lambda_{\max}(\H_{\textup{filt}}) &\le \max\left(
		2\lambda + n L^{*}, \, 2\xi\right).
	\end{align}
	This holds for all $\A,\B,\bgamma$ such that $\rank([\A,\B])\le r$ and under the convex constraint (also recall that $\U$ is the vertical stack of $\A$ and $\bgamma$). Hence we conclude that the objective function $F$ in \eqref{eq:SMF_filt_CALE0} verifies RSC and RSM properties (Def. \ref{def:RSC}) with parameters $\mu=\min(n \mu^{*}+2\lambda, \, 2\xi)$ and $L=\max(n L^{*}+ 2\lambda, \, 2\xi)$. 
	This verifies \eqref{eq:RSC_thm_pf} for the chosen parameters $\mu$ and $L$. 	Then the rest follows from Theorem \ref{thm:CALE_LPGD}. 
\end{proof}

Next, we prove Theorem \ref{thm:SMF_LPGD_full} for SMF-$\H$, the exponential convergence of Algorithm \ref{alg:SMF_PGD} for the feature-based SMF.

\begin{proof}[\textbf{Proof of Theorem \ref{thm:SMF_LPGD_full} for SMF-$\H$}]

	We will use the same setup as in the proof of Theorem \ref{thm:SMF_LPGD_full} for SMF-$\W$. The main part of the argument is the computation of the Hessian of loss function $f:=F$ for SMF-$\H$ in \eqref{eq:SMF_feat_CALE0}, which is straightforward but substantially more involved than the corresponding computation for the filter-based case in the proof of Theorem \ref{thm:SMF_LPGD_full}. Let $\a_{s}=\A[:,s]+\bgamma^{T}\x'_{s}$ for the feature-based model with $\A\in \R^{\kappa\times n}$. Denote 
	\begin{align}
		\a_{s}=\I_{\kappa} \A[:,s]+\bgamma^{T}\x_{s}' =: \bigg[ \left\langle \underbrace{\begin{bmatrix} 
				\I_{\kappa}[:,j] \\ \bgamma[:,j] \end{bmatrix}}_{=:\v_{j}} ,\, \underbrace{\begin{bmatrix} 
				\A[:,s] \\ \x_{s}' \end{bmatrix}}_{=:\bpsi_{s}}  \right\rangle; \,\, j=1,\dots,\kappa \bigg]^{T}\in \R^{\kappa}.
	\end{align}
	Note that in the above representation we have concatenated $\A[:,s]$ with the auxiliary covariate $\x'_{s}$, whereas previously for SMF-$\W$ (see \eqref{eq:thm_SMF_filt_as_def}), we concatenated $\A[:,j]$ with regression coefficient $\bgamma[:,j]$ for the auxiliary covarate for the $j$th class\footnote{This is because for the feature-based model, the column $\A[:,s]\in \R^{\kappa}$ for $s=1,\dots,n$ represent a feature of the $s$th sample, whereas for the filter-based model, $\A[:,j]$ for $j=1,\dots,\kappa$ represents the $j$th filter that is applied to the feature $\x_{s}$ of the $s$th sample.}. A straightforward computation shows the following gradient formulas:
	\begin{align}\label{eq:SMF_feat_gradients}
		&\nabla_{\vect(\A)}  f(\A,\bgamma,\B) = \left(  \sum_{s=1}^{n} \dot{\h}(y_{s},\a_{s}) \otimes \I_{n}[:,s] \right)  + 2\lambda \vect(\A)
		=
		\begin{bmatrix} 
			\dot{\h}(y_{1},\a_{1}) \\
			\vdots  \\
			\dot{\h}(y_{n},\a_{n})
		\end{bmatrix}  
		+ 2\lambda \vect(\A), \\
		&\nabla_{\vect(\bgamma)}  f(\A,\bgamma,\B) = \left(  \sum_{s=1}^{n} \dot{\h}(y_{s},\a_{s}) \otimes \x'_{s} \right) + 2\lambda \vect(\bgamma), \\
		&\nabla_{\B}  f(\A,\bgamma,\B) = 2\xi(\B-\X_{\textup{data}}) \\ 	&\nabla_{\vect(\A)}\nabla_{\vect(\A)^{T}}  f(\A,\bgamma,\B) = \diag\left( \ddot{\H}(y_{1},\a_{1}),\dots, \ddot{\H}(y_{n},\a_{n}) \right)  + 2\lambda \I_{\kappa n}, \\
		&\nabla_{\vect(\bgamma)}\nabla_{\vect(\bgamma)^{T}}  f(\A,\bgamma,\B) =  \left(  \sum_{s=1}^{n} \ddot{\H}(y_{s},\a_{s}) \otimes \x_{s}'(\x_{s}')^{T}\right) + 2\lambda \I_{q\kappa}, \\
		&\nabla_{\vect(\B)}\nabla_{\vect(\B)^{T}}  f(\A,\bgamma,\B) = 2\xi \I_{pn}, \\            &\nabla_{\vect(\bgamma)}\nabla_{\vect(\A)^{T}}  f(\A,\bgamma,\B) =  \left[ \ddot{\H}(y_{1},\a_{1})\otimes \x'_{1}, \dots, \ddot{\H}(y_{n},\a_{n})\otimes \x'_{n} \right] \in \R^{\kappa q \times \kappa n}, \\
		&\nabla_{\vect(\B)}\nabla_{\vect(\V)^{T}}  f(\A,\bgamma,\B) = O \qquad \text{ for } \V = \A, \bgamma.
	\end{align}
	From this, we will compute the eigenvalues of the Hessian $\H_{\textup{feat}}$ of the loss function $f$. In order to illustrate our computation in a simple setting, we first assume $\kappa=1=q$, which corresponds to binary classification $\kappa=1$ with one-dimensional auxiliary features $q=1$. In this case, we have 
		{\small
	\begin{align}
		\H_{\textup{feat}}&:=\nabla_{\vect(\A,\bgamma,\B)} \nabla_{\vect(\A,\bgamma,\B)^{T}} f(\A,\bgamma,\B) \\
		&\hspace{-1cm} =
			\begin{bmatrix}
				\ddot{h}(y_{1},\a_{1})+2\lambda  & 0  & \dots & 0 &  \ddot{h}(y_{1},\a_{1}) x_{1}'  & O \\
				0 & \ddot{h}(y_{2},\a_{2})+2\lambda   &\dots & 0 &  \ddot{h}(y_{2},\a_{2}) x_{2}'  & O \\
				\vdots & \vdots & \ddots & \vdots & \vdots & \vdots \\
				0 & \dots & 0 &  \ddot{h}(y_{n},\a_{n})+2\lambda  &  \ddot{h}(y_{n},\a_{n}) x_{n}' & O \\
				\ddot{h}(y_{1},\a_{1}) x_{1}' & \ddot{h}(y_{2},\a_{2}) x_{2}' & \dots & \ddot{h}(y_{n},\a_{n}) x_{n}' & \left( \sum_{s=1}^{n}\ddot{h}(y_{s},\a_{s}) (x_{s}')^{2} \right) + 2\lambda   & O \\
				O&O&\dots &O &O & 2\xi \I_{pn} 
			\end{bmatrix},
	\end{align}
	}
	where we denoted $\ddot{h}=\ddot{h}_{11}\in \R$ and $x'_{s}=\x_{s}'\in \R$ for $s=1,\dots,n$. In order to compute the eigenvalues of the above matrix, we will use the following formula for the determinant of $3\times 3$ block matrix: ($O$ representing matrices of zero entries with appropriate sizes)
	\begin{align}
		\det\left( \begin{bmatrix} 
			A & B & O \\
			B^{T} & C & O \\ 
			O & O & D
		\end{bmatrix} \right) 
		= \det\left( C - B^{T}A^{-1}B \right) \det(A) \det(D).
	\end{align}
	This yields the following simple formula for the characteristic polynomial of $\H_{\textup{feat}}$:
	\begin{align}
		&\det( \H_{\textup{feat}} - t \I) \\
		&= \left(  \sum_{s=1}^{n} \ddot{h}(y_{s},\a_{s}) (x_{s}')^{2}  -  \sum_{s=1}^{n} \frac{(\ddot{h}(y_{s},\a_{s}))^{2} (x_{s}')^{2}}{\ddot{h}(y_{s},\a_{s})+2\lambda-t } + 2\lambda  - t \right) \prod_{s=1}^{n} \left( \ddot{h}(\y_{s},\a_{s}) + 2\lambda- t \right)(2\xi - t)^{pn} \\
		&=\left( \sum_{s=1}^{n}  \frac{ (2\lambda -t) \ddot{h}(y_{s},\a_{s}) (x_{s}')^{2}}{ \ddot{h}(y_{s},\a_{s})+2\lambda -t} + 2\lambda - t \right)  \prod_{s=1}^{n} \left( \ddot{h}(\y_{s},\a_{s}) + 2\lambda- t\right) (2\xi - t)^{pn}.
	\end{align}
	Since the first term in the parenthesis in the above equation has solution $2\lambda$, it follows that 
	\begin{align}
		\lambda_{\min}(\H_{\textup{feat}}) & \ge \min(2\lambda, \, \alpha^{-}+2\lambda, \, 2\xi) = \min(2\lambda, \, 2\xi), \\
		\lambda_{\max}(\H_{\textup{feat}}) &\le \max\left(2\lambda,\, \alpha^{+} + 2\lambda, \, 2\xi \right)=\max\left(\alpha^{+} + 2\lambda, \, 2\xi \right).
	\end{align}
	
	Now we generalize the above computation for the general case $\kappa,q\ge 1$. First, note the general form of the Hessian as below:
	{\small
		\begin{align}
			&\H_{\textup{feat}}:=\nabla_{\vect(\A,\bgamma,\B)} \nabla_{\vect(\A,\bgamma,\B)^{T}} f(\A,\bgamma,\B) \\
			&\hspace{-0.5cm} =
			\begin{bmatrix}
				\ddot{\H}(y_{1},\a_{1})+2\lambda  \I_{\kappa} & 0  & \dots & 0 & (\ddot{\H}(y_{1},\a_{1})\otimes \x_{1}')^{T} & O \\
				0 & \ddot{\H}(y_{2},\a_{2})+2\lambda  \I_{\kappa}  &\dots & 0 & (\ddot{\H}(y_{2},\a_{2})\otimes \x_{2}')^{T} & O \\
				\vdots & \vdots & \ddots & \vdots & \vdots & \vdots \\
				0 & \dots & 0 &  \ddot{\H}(y_{n},\a_{n})+2\lambda  \I_{\kappa} & (\ddot{\H}(y_{n},\a_{n})\otimes \x_{n}')^{T} & O \\
				\ddot{\H}(y_{1},\a_{1})\otimes \x_{1}' & \ddot{\H}(y_{2},\a_{2})\otimes \x_{2}' & \dots & \ddot{\H}(y_{n},\a_{n})\otimes \x_{n}' & \begin{matrix}  \sum_{s=1}^{n}\ddot{\H}(y_{s},\a_{s}) \otimes \x_{s}'(\x_{s}')^{T}  \\  + 2\lambda  \I_{q\kappa}   \end{matrix} & O \\
				O&O&\dots &O &O & 2\xi  \I_{pn} 
			\end{bmatrix}.
		\end{align}
	}
	Note that for any square symmetric matrix $B$ and a column vector $\x$ of matching size, 
	\begin{align}
		B\otimes \x \x^{T} - (B\otimes \x)^{T}(B+w \I)^{-1} (B\otimes \x) &=\left( B-B(B+w \I)^{-1} B \right) \otimes (\x\x^{T})  \\
		&= \left( (B+wI)-B \right) (B+w \I)^{-1} B\otimes (\x\x^{T})  \\    
		&= w(B+w I)^{-1} B \otimes \x\x^{T}. 
	\end{align}
	Hence by a similar computation as before, we obtain
	\begin{align}
		&\det( \H_{\textup{feat}} - t \I)  \\
		& = \det\left( (2\lambda-t)\sum_{s=1}^{n} \left( \ddot{\H}(y_{s},\a_{s}) + (2\lambda  -t) \I_{\kappa} \right)^{-1} \ddot{\H}(y_{s},\a_{s}) \otimes \x_{s}'(\x_{s}')^{T}  + (2\lambda -t) \I_{q\kappa}  \right)  \\
		&\hspace{5cm} \times \left(  \prod_{s=1}^{n} \det\left( \ddot{\H}(y_{s},\a_{s}) + (2\lambda - t)\I_{\kappa} \right)\right) (2\xi  - t)^{pn}.
	\end{align}
	It follows that 
	\begin{align}
		\lambda_{\min}(\H_{\textup{feat}}) & \ge \min( 2\lambda, \, 2\xi), \label{eq:SMF_keng} \\
		\lambda_{\max}(\H_{\textup{feat}}) &\le \max\left(
		\alpha^{+} + 2\lambda, \, 2\xi \right) \label{eq:SMF_thm_proof}.
	\end{align}
	Then the rest follows from Theorem \ref{thm:CALE_LPGD}.
\end{proof}

\subsection{Proof of Theorem \ref{thm:SMF_LPGD_STAT}}

In this section, we prove the statistical estimation guarantee for SMF in Theorem \ref{thm:SMF_LPGD_STAT}. Recall the generative model for SMF in \eqref{eq:SMF_generative}.  Our proof is based on Theorem \ref{thm:SMF_LPGD} that we have established previously and standard matrix concentration bounds, which we provide below:

\begin{lemma}[2-norm of matrices with bounded and independent columns]
	\label{lem:bernstein_app}
	Let $\X$ be a $d_{1}\times d_{2}$ random matrix of independent, mean zero, real-valued columns such that $\lVert \X \rVert_{\infty}<L$ almost surely for some constant $L>0$. Then 
	\begin{align}
		\P\left( \left\lVert \X \right\rVert_{2} \ge t \right) \le (d_{1}+d_{2}) \exp\left( \frac{-t^{2}/2}{\max\{d_{1}, d_{2}\} L^{2} + (Lt/3)} \right).
	\end{align}
\end{lemma}

\begin{proof}
	This lemma is a simple consequence of the matrix Burnstein's inequality (see, e.g., \cite[Thm. 6.1.1]{tropp2015introduction}). Indeed, write $\X=[\mathbf{x}_{1},\dots,\mathbf{x}_{d_{2}}]$, where $\mathbf{x}_{j}$s are the columns of $\X$. Note that, by the hypothesis, 
	\begin{align}
		\lVert \E[\X^{T}\X] \rVert_{2} =  \lVert \textup{diag}( \E[\lVert \mathbf{x}_{1} \rVert^{2}_{2}],\dots, \E[\lVert \mathbf{x}_{d_{2}} \rVert^{2}_{2}]) \rVert_{2} \le d_{1} L^{2}. 
	\end{align}
	Similarly, 
	\begin{align}
		\lVert \E[\X \X^{T}] \rVert_{2} = \left\lVert \sum_{j=1}^{d_{2}} \E[\mathbf{x}_{j} \mathbf{x}_{j}^{T}]  \right\rVert_{2} \le d_{2} L^{2}. 
	\end{align}
	It follows that matrix variance statistics $v(\X)$ of $\X$ satisfies 
	\begin{align}
		v(\X)&:= \max\left\{ \lVert \E[\X \X^{T}] \rVert_{2},\, \lVert \E[\X^{T} \X] \rVert_{2} \right\} \le \max\{d_{1}, d_{2}\} L^{2}.
	\end{align}
	Then the tail bound on the 2-norm of $\X$ as asserted follows immediately from the matrix Burnstein's inequality. 
\end{proof}

\begin{lemma}(2-norm of matrices with independent sub-gaussian entries)
	\label{lem:matrix_norm_bd}
	Let $\A$ be an $m\times n$ random matrix with independent subgaussian entries $\A_{ij}$ of mean zero. Denote $K$ to be the maximum subgaussian norm of $\A_{ij}$, that is, $K>0$ is the smallest number such that $\E[\exp(\A_{ij})^{2}/K^{2}  ]\le 2$. Then for each $t>0$,
	\begin{align}
		\P\left( \lVert \A \rVert_{2} \ge 3K(\sqrt{m}+\sqrt{n}+t) \right) \le 2\exp(-t^{2}).
	\end{align}
\end{lemma}

\begin{proof}
	See Theorem 4.4.5 in \cite{vershynin2018high}.
\end{proof}

Now we prove Theorem \ref{thm:SMF_LPGD_STAT} for SMF-$\W$.

Recall that the ($L_{2}$-regularized) negative log-likelihood of observing triples $(y_{i},\x_{i},\x_{i}')$ for $i=1,\dots, n$ is given as 
\begin{align}\label{eq:SMF_likelihood_conv_filter}
	\L_{n} &  := F(\A,\B,\bgamma)  + \frac{1}{2(\sigma')^{2}} \lVert \X_{\textup{aux}}-\C \rVert_{F}^{2} + c,
\end{align}
where $c$ is a constant and $F$ is as in \eqref{eq:SMF_feat_CALE0} or \eqref{eq:SMF_filt_CALE0} depending on the activation type with tuning parameter  $\xi=\frac{1}{2\sigma^{2}}$. Write the true parameter $\bZ^{\star}=[\param^{\star}, \bgamma^{\star}]$. Recall that $\rank(\param^{\star})\le r$ by the model assumption in \eqref{eq:SMF_generative}. 

\begin{proof}[\textbf{Proof of Theorem \ref{thm:SMF_LPGD_STAT} for SMF-$\W$}]
	Let us define the expected loss function {\color{black}$\bar{F}(\A,\B,\bgamma):= \E_{\beps_{i},\beps_{i}',1\le i \le n}\left[ F(\A,\B,\bgamma)\right]$}. Define the following gradient mappings of $\bZ^{\star}$ with respect to the {\color{black}empirical $F$ and the expected $\bar{F}$ loss functions}: 
	\begin{align}
		G(\bZ^{\star}, \tau) :=\frac{1}{\tau}\left( \bZ^{\star} - \Pi_{\Param} \left( \bZ^{\star} - \tau \nabla F(\bZ^{\star}) \right)\right), \quad  \bar{G}(\bZ^{\star},\tau):=\frac{1}{\tau}\left( \bZ^{\star} - \Pi_{\Param} \left( \bZ^{\star} - \tau \nabla \bar{F}(\bZ^{\star}) \right)\right).
	\end{align}
	It is elementary to show that the true parameter $\bZ^{\star}$ is a stationary point of $\bar{F}- \lambda (\lVert \A\rVert_{F}^{2} + \lVert \bgamma \rVert_{F}^{2}) $ over $\Param\subseteq \R^{p\times (\kappa+n)}\times \R^{q\times \kappa}$. Hence we have $\bar{G}(\bZ^{\star},\tau)= 2\lambda[\A^{\star},O,\bgamma^{\star}]$, so we may write 
	\begin{align}\label{eq:grad_mapping_compare_stationary}
		G(\bZ^{\star}, \tau) &= G(\bZ^{\star}, \tau) - \bar{G}(\bZ^{\star}, \tau) + 2\lambda[\A^{\star},O,\bgamma^{\star}] \\
		&= \frac{1}{\tau}\left[  \Pi_{\Param}\left( \bZ^{\star}-\tau\nabla F(\bZ^{\star}) \right) - \Pi_{\Param}\left( \bZ^{\star}-\tau\nabla \bar{F}(\bZ^{\star}) \right) \right] +  2\lambda[\A^{\star},O,\bgamma^{\star}]
	\end{align}
	We will consider two cases, depending on whether the true parameter $\mathbf{Z}^{\star}$ satisfies the first-order optimality condition for $f$ over the convex constraint $\Param$. The first-order optimality w.r.t. the low-rank constraint is handled directly by Theorem \ref{thm:SMF_LPGD_full}. 
	
	\textbf{Case 1.} $\bZ^{\star}-\tau \nabla  F(\bZ^{\star})\in \Param$ (In particular, this is the case where $\Param$ equals the whole space). 
	
	In this case, we can disregard the projection $\Pi_{\Param}$ in the above display so we get 
	\begin{align}
		G(\bZ^{\star}, \tau) - 2\lambda[\A^{\star},O,\bgamma^{\star}] = \nabla F(\bZ^{\star}) - \nabla \bar{F}(\bZ^{\star}) =: [\Delta \widetilde{\param}^{\star}, \Delta \widetilde{\bgamma}^{\star}].
	\end{align}
	We will show that,  for some constants $c,C>0$, with probability at least $1-n^{-1}$,
	\begin{align}\label{eq:stat_SFM_W_claim1}
		S:=\sqrt{3r} \lVert \Delta \widetilde{\param}^{\star}  \rVert_{2}  + \lVert \Delta \widetilde{\bgamma}^{\star} \rVert_{F} \le c \sqrt{n} \log n  + 3C\sigma(\sqrt{p} + \sqrt{n}+ c\sqrt{\log n}). 
	\end{align}
	By Theorem \ref{thm:SMF_LPGD_full} with $[\Delta\param^{\star}, \Delta \bGamma^{\star}] :=G(\bZ^{\star}, \tau)$, 
	\begin{align}
		&\lVert \bZ_{t} - \bZ^{\star}  \rVert_{F} -  \rho^{t}  \, \lVert  \bZ_{0} - \bZ^{\star}\rVert_{F} \\
		&\qquad \le \frac{\tau}{1-\rho} \left( \sqrt{3r} \lVert \Delta \param^{\star} \rVert_{2}   +   \lVert \Delta \bgamma^{\star} \rVert_{F} \right) \\
		&\qquad \le \frac{\tau}{1-\rho} \left( \sqrt{3r}( \lVert \Delta \widetilde{\param}^{\star} \rVert_{2} + 2\lambda \lVert \A^{\star} \rVert_{2}  )  + (  \lVert \Delta \widetilde{\bgamma}^{\star} \rVert_{F}+2\lambda\lVert \bgamma^{\star} \rVert_{F}) \right).
	\end{align}
	It follows that with probability at least $1-n^{-1}$, 
	\begin{align}
		\lVert \bZ_{t} - \bZ^{\star}  \rVert_{F} 
		& \le  \rho^{t}  \, \lVert  \bZ_{0} - \bZ^{\star}\rVert_{F} + \frac{\tau}{1-\rho}\left( c \sqrt{n} \log n  + 3C\sigma(\sqrt{p} + \sqrt{n}+ c\sqrt{\log n})  \right) \\
		&\qquad + \frac{2\lambda\tau}{1-\rho}\left(  \sqrt{3r}\lVert \A^{\star} \rVert_{2} + \lVert \bgamma^{\star} \rVert_{F} \right).
	\end{align}
	Now since $L/\mu<3$ and $\tau\in (\frac{1}{2\mu}, \frac{3}{2L})$, there exists $\eps>0$ such that $\tau=\frac{1}{(2-\eps)\mu}$. Then $\frac{\tau}{1-\rho} = \frac{\tau}{2\tau \mu - 1}=\frac{1}{\eps \mu}$. Thus, with probability at least $1-n^{-1}$, 
	\begin{align}
		\lVert \bZ_{t} - \bZ^{\star}  \rVert_{F} - \rho^{t}  \, \lVert  \bZ_{0} - \bZ^{\star}\rVert_{F} \le O\left( \frac{\sqrt{n} \log n + \lambda}{\mu}\right),
	\end{align}
	as desired.

	Now we show \eqref{eq:stat_SFM_W_claim1}. The argument is that, the norm of $[\Delta \widetilde{\param}^{\star}, \Delta \widetilde{\bgamma}^{\star}]$ can be decomposed into the sum of norms of random matrices with independent mean zero columns or mean zero Gaussian random matrices, which should have norm at most $\sqrt{n}\log n$ with high probability by standard matrix concentration inequalities. 
	
	We use the notation $\U=[\A^{T}, \bgamma^{T}]^{T}$, $\U^{\star}=[(\A^{\star})^{T}, (\bgamma^{\star})^{T}]^{T}$, $\bPhi=[\bphi_{1},\dots,\bphi_{n}]=[\X_{\textup{data}}^{T}, \X_{\textup{aux}}^{T}]^{T}$ (see also the proof of Theorem \ref{thm:SMF_LPGD}). Denote $\a_{s}=\U^{T}\bphi_{s}$ and $\a_{s}^{\star}=(\U^{\star})^{T}\bphi_{s}$ for $s=1,\dots,n$ and introduce the following random quantities 
	\begin{align}\label{eq:def_Q}
		&\mathtt{Q}_{1}:= \begin{bmatrix} 
			\dot{\h}(y_{1},\a_{1}),\dots,
			\dot{\h}(y_{n},\a_{n})
		\end{bmatrix}  \in \R^{\kappa \times n} , \\
		&\mathtt{Q}_{2}:= [ \beps_{1},\dots,\beps_{n}]  \in \R^{p\times n},\quad  
		\mathtt{Q}_{3}:= [ \beps_{1}',\dots,\beps_{n}']  \in \R^{p\times n}.
	\end{align}

	Recall that 
	\begin{align}
		&\nabla_{\vect(\U)}  F(\U,\B) = \left( \sum_{s=1}^{n} \dot{\h}(y_{s},\a_{s}) \otimes \bphi_{s} \right) + 2\lambda \vect(\U), \quad \nabla_{\B}  F(\U,\B) = \frac{2}{2\sigma^{2}} (\B-\X_{\textup{data}}), \\ 	
		&\nabla_{\vect(\U)}  \bar{F}(\U,\B) = \left(  \sum_{s=1}^{n} \E\left[ \dot{\h}(y_{s},\a_{s}) \otimes \bphi_{s} \right] \right) + 2\lambda \vect(\U), \quad \nabla_{\B}  \bar{F}(\U,\B) = \frac{2}{2\sigma^{2}} (\B-\B^{\star}),
	\end{align}
	where $\dot{\h}$ is defined in \eqref{eq:Hddot_def}. Note that 
	\begin{align}
		\E\left[ \dot{\h}(y_{s},\a_{s}) \,\bigg|\, \bphi_{s} \right]  		&= \left[ \left(\frac{h'(\a[j])}{1+\sum_{c=1}^{\kappa} h(\a[c])} - g_{j}(\a_{s}^{\star})\frac{h'(\a[j])}{h(\a[j])}  \right)_{\a=\a_{s}} \, ; \, j=1,\dots,\kappa \right]\\
		&\hspace{-1cm} = \left[ \left(\frac{h'(\a[j])}{1+\sum_{c=1}^{\kappa} h(\a[c])} - \frac{h(\a_{s}^{\star}[j])}{1+\sum_{c=1}^{\kappa} h(\a_{s}^{\star}[c])} \frac{h'(\a[j])}{h(\a[j])}  \right)_{\a=\a_{s}} \, ; \, j=1,\dots,\kappa \right],
	\end{align}
	so the above vanishes when $\a_{s}=\a_{s}^{\star}$. Hence 
	\begin{align}\label{eq:dot_h_exp_vanish}
		\E\left[ \dot{\h}(y_{s},\a_{s}^{\star}) \otimes \bphi_{s} \right] = \E\left[ \E\left[ \dot{\h}(y_{s},\a_{s}^{\star}) \otimes \bphi_{s}\,\bigg|\, \bphi_{s} \right]  \right] =\mathbf{0},
	\end{align}
	Hence we can compute the following gradients 
	\begin{align}
		\nabla_{\vect(\A)} (F - \bar{F})(\A,\B,\bgamma)   &=\left(  \sum_{s=1}^{n} \dot{\h}(y_{s},\a_{s}) \otimes \x_{s} \right)  \\
		\nabla_{\vect(\bgamma)} (F - \bar{F})(\A,\B,\bgamma) &=\left( \sum_{s=1}^{n} \dot{\h}(y_{s},\a_{s}) \otimes \x_{s}' \right) \\
		\nabla_{\B}  (F - \bar{F})(\A,\B,\bgamma) &=\frac{2}{2\sigma^{2}} (\B^{\star}-\X_{\textup{data}}) = \frac{2}{2\sigma^{2}} [ \beps_{1},\dots,\beps_{n}] .
	\end{align}
	It follows that (recall the definition of $\gamma_{\max}$ in Assumption \ref{assumption:A4}) 
	\begin{align}
		\lVert \nabla_{\A} (F - \bar{F})(\A^{\star},\B^{\star},\bgamma^{\star})  \rVert_{2}  &= \left\rVert\sum_{s=1}^{n} (\B^{\star}[:,s]+\beps_{s}) \dot{\h}(y_{s},\a_{s}^{\star})^{T} \right\rVert_{2}  \\
		& \le \left\rVert  \sum_{s=1}^{n} \B^{\star}[:,s] \dot{\h}(y_{s},\a_{s}^{\star})^{T} \right\rVert_{2}  +\left\rVert \sum_{s=1}^{n} \beps_{s} \dot{\h}(y_{s},\a_{s}^{\star})^{T} \right\rVert_{2}  \\
		&\le  \lVert \B^{\star}\rVert_{\infty} \left\rVert  \mathtt{Q}_{1} \right\rVert_{2}  +  \gamma_{\max}  \left\rVert \mathtt{Q}_{2}  \right\rVert_{2}. 
	\end{align}
	Similarly, we have 
	
	\begin{align}
		\lVert \Delta \widetilde{\bgamma}^{\star} \rVert_{F}=	\lVert \nabla_{\bgamma} (F - \bar{F})(\A^{\star},\B^{\star},\bgamma^{\star})  \rVert_{F} &  \le \sqrt{q} \lVert \nabla_{\vect(\bgamma)} (F - \bar{F})(\A^{\star},\B^{\star},\bgamma^{\star})  \rVert_{2} \\
		&\le  \sqrt{q} \lVert \C^{\star}\rVert_{\infty}  \left\rVert \mathtt{Q}_{1}  \right\rVert_{2}  + \sqrt{q} \gamma_{\max}  \left\rVert \mathtt{Q}_{3}  \right\rVert_{2}.
	\end{align}

	Using the fact that $\lVert [A,B] \rVert_{2}\le \lVert A \rVert_{2} +  \lVert B \rVert_{2} $ for two matrices $A,B$ with the same number of rows, we have 
	\begin{align}\label{eq:SMF_MLE_pf_bd_Q}
		\left\rVert \Delta \widetilde{\param}^{\star}  \right\rVert_{2} &= \left\lVert  \nabla_{\A} (F - \bar{F})(\A^{\star},\B^{\star},\bgamma^{\star})   \right\rVert_{2}  + \left\lVert  \nabla_{\B} (F - \bar{F})(\A^{\star},\B^{\star},\bgamma^{\star})   \right\rVert_{2}  \\
		&\le  \lVert \B^{\star}\rVert_{\infty} \left\rVert \mathtt{Q}_{1} \right\rVert_{2}  + \gamma_{\max}  \left\rVert \mathtt{Q}_{2} \right\rVert_{2} +\frac{2}{2\sigma^{2}} \left\lVert  \mathtt{Q}_{2} \right\rVert_{2}. 
	\end{align}
	Thus, combining the above bounds, we obtain 
	\begin{align}\label{eq:SMF_MLE_pf_bd_Q2}
		S= 
		\sqrt{3r} \lVert \Delta \widetilde{\param}^{\star} \rVert_{2} + \lVert \Delta \widetilde{\bgamma}^{\star} \rVert_{F}  \le  \sum_{i=1}^{3} c_{i} \lVert \mathtt{Q}_{i} \rVert_{2},
	\end{align}
	where the constants $c_{1},c_{2},c_{3}>0$ are given by 
	\begin{align}\label{eq:c_constants_Q}
		c_{1}:=\left( \sqrt{3r} \lVert \B^{\star}\rVert_{\infty} + \sqrt{q}\lVert \C^{\star} \rVert_{\infty} \right), \quad 
		c_{2}:=\sqrt{3r} (\gamma_{\max} + \frac{2}{{2\sigma^{2}}}),\quad 
		c_{3}:=\sqrt{q}\gamma_{\max}.
	\end{align}

	Next, we will use concentration inequalities to argue that the right hand side in \eqref{eq:SMF_MLE_pf_bd_Q2} is small with high probability and obtain the following tail bound on $S$: 
	\begin{align}\label{eq:S_tail_bound}
		\P\left(S>c \sqrt{n} \log n  + 3C\sigma(\sqrt{p} + \sqrt{n}+ c\sqrt{\log n})   \right)  \le \frac{1}{n},
	\end{align}
	where $C>0$ is an absolute constant and $c>0$ can be written explicitly in terms of the constants we use in this proof. This is enough to conclude \eqref{eq:stat_SFM_W_claim1}. 
	
	Recall that for a random variable $Z$, its sub-Gaussian norm, denoted as $\lVert Z \rVert_{\psi_{2}}$, is the smalleset number $K>0$ such that $\E[\exp(Z^{2}/K^{2})]\le 2$. The constant $C>$ above is the sub-gaussian norm of the standard normal variable, which can be taken as $C\le 36e/\log 2$. Using union bound with Lemmas \ref{lem:bernstein_app} and \ref{lem:matrix_norm_bd}, for each $t,t'>0$, we get 
	\begin{align}\label{eq:S_bd_pf}
		&\P\left( S > c_{1} t  + 3(c_{2}+c_{3})C\sigma(\sqrt{p} + \sqrt{n}+ t')   \right)  \\
		&{\color{black}\qquad \le \P\left( \lVert \mathtt{Q}_{1} \rVert_{2}>t\right)+ 
			\left( \sum_{i=2}^{3} \P\left( \lVert \mathtt{Q}_{i} \rVert_{2}> \frac{3C\sigma}{2}(\sqrt{p} + \sqrt{n}+ t')\right)  \right)}\\ 
		&\qquad \le 2\kappa \exp\left( \frac{-t^{2} }{C_{1}^{2} \kappa^{2} n} \right) + {\color{black}2\exp(-(t')^{2})}.
	\end{align}
	Indeed, for bounding $\P(\lVert \mathtt{Q}_{1} \rVert_{2}>t)$, we used Lemma \ref{lem:bernstein_app}; for bounding tail probabilities of $\lVert \mathtt{Q}_{2} \rVert_{2}$ and $\lVert \mathtt{Q}_{3} \rVert_{2}$, 
		we used Lemma \ref{lem:matrix_norm_bd} with $K=\frac{C\sigma}{2}$ and $K=\frac{C\sigma'}{2}$, respectively. Observe that in order to make the last expression in \eqref{eq:S_bd_pf} small, we will chose $t=c_{4}\sqrt{n}\log n$ and $t'=c_{4}\sqrt{\log n}$, where $c_{4}>0$ is a constant to be determined. This yields
		\begin{align}
			\P\left( S> c_{1}c_{4} \sqrt{n} \log n  + 3(c_{2}+c_{3})C\sigma(\sqrt{p} + \sqrt{n}+ c_{4}\sqrt{\log n})   \right) \le  n^{-c_{5}},
		\end{align}
		where $c_{5}>0$ is an explicit constant that grows in $c_{4}$. We assume $c_{4}>0$ is such that $c_{5}\ge 1$. This shows \eqref{eq:S_tail_bound}.

		\textbf{Case 2.} $\bZ^{\star}-\tau \nabla F(\bZ^{\star})\notin \Param$.
		
		In this case,  we cannot directly simplify the expression \eqref{eq:grad_mapping_compare_stationary}. In this case, we take the Frobenius norm and use non-expansiveness of the projection operator (onto convex set $\Param$): 
		\begin{align}\label{eq:grad_mapping_compare_stationary2}
			\lVert G(\bZ^{\star}, \tau) {\color{black} -\bar{G}(\bZ^{\star}, \tau)} \rVert_{F} &=  \frac{1}{\tau} \left\rVert\left[  \Pi_{\Param}\left( \bZ^{\star}-\tau\nabla F(\bZ^{\star}) \right) - \Pi_{\Param}\left( \bZ^{\star}-\tau\nabla \bar{F}(\bZ^{\star}) \right) \right] \right\rVert_{F}   \\
			&\le \lVert \nabla F(\bZ^{\star})-  \nabla \bar{F}(\bZ^{\star}) \rVert_{F} \\
			& \le  \lVert \Delta \widetilde{\param}^{\star} \rVert_{F} + \lVert \Delta \widetilde{\bgamma}^{\star} \rVert_{F}. 
		\end{align}
		According to Remark \ref{rmk:pf_thm_LPGD}, we also have Theorem \ref{thm:CALE_LPGD} (and hence Theorem \ref{thm:SMF_LPGD}) with $\sqrt{3r} \lVert \Delta \widetilde{\param}^{\star} \rVert_{2}$ replaced with $\lVert \Delta \widetilde{\param}^{\star} \rVert_{F}$. Then an identical argument with $\lVert \mathtt{Q}_{i}\rVert_{F} \le \sqrt{\min(p,n)} \lVert \mathtt{Q}_{i}\rVert_{2}$ for $i=2,3$ shows 
		\begin{align}\label{eq:SMF_MLE_pf_bd_Q4}
			S':= \lVert \Delta \widetilde{\param}^{\star} \rVert_{F} + \lVert \Delta \widetilde{\bgamma} \rVert_{F} \le c_{1} \lVert \mathtt{Q}_{1} \rVert_{2}  + c_{2} \sqrt{\min( p,n )} \lVert \mathtt{Q}_{2}  \rVert_{2} + c_{3} \sqrt{\min( p,n )} \lVert \mathtt{Q}_{3} \rVert_{2},
		\end{align}
		where the constants $c_{1},c_{2},c_{3}>0$ are the same as in \eqref{eq:c_constants_Q}. So we have 
		\begin{align}
			\lVert \bZ_{t} - \bZ^{\star} \rVert_{F} \le \rho^{t}  \, \lVert  \bZ_{0} - \bZ^{\star}\rVert_{F} + \frac{\tau}{1-\rho}(S' + 2\lambda (\lVert \A^{\star} \rVert_{2}+\lVert \bgamma^{\star} \rVert_{F})). 
		\end{align}
		Then an identical argument shows 
		\begin{align}\label{eq:S_bd_pf2}
			&\P\left( S' > c_{1} t  + 3(c_{2}+c_{3})C\sigma(\sqrt{p} + \sqrt{n}+ t')  \sqrt{\min(p,n)}  \right) \\
			&\qquad \le  \P\left( \lVert \mathtt{Q}_{1} \rVert_{2} >t\right)  + \sum_{i=2}^{3} \P\left( \lVert \mathtt{Q}_{i} \rVert_{2} > \frac{3C\sigma}{2}(\sqrt{p} + \sqrt{n}+ t')    \right),
		\end{align}
		and the assertion follows similarly as before. 
	\end{proof}

	It remains to show Theorem \ref{thm:SMF_LPGD_STAT} for SMF-$\H$.

	\begin{proof}[\textbf{Proof of Theorem \ref{thm:SMF_LPGD_STAT} for SMF-$\H$}]
		The argument is entirely similar to the proof of Theorem \ref{thm:SMF_LPGD_STAT} for SMF-$\W$. Indeed, denoting $\a_{s}=\A[:,s]+\bgamma^{T}\x_{s}'$ for $s=1,\dots,n$ and keeping the other notations the same as in the proof of Theorem \ref{thm:SMF_LPGD_STAT}, we can compute the following gradients 
		\begin{align}
			\nabla_{\A} (F - \bar{F})(\A,\B,\bgamma)   &=  \begin{bmatrix} 
				\dot{\h}(y_{1},\a_{1}),\dots,
				\dot{\h}(y_{n},\a_{n})
			\end{bmatrix}  \\
			\nabla_{\vect(\bgamma)} (F - \bar{F})(\A,\B,\bgamma) &=\left( \sum_{s=1}^{n} \dot{\h}(y_{s},\a_{s}) \otimes \x_{s}' \right) \\
			\nabla_{\B}  (F - \bar{F})(\A,\B,\bgamma) &=\frac{2}{2\sigma^{2}} (\B^{\star}-\X_{\textup{data}}) = \frac{2}{2\sigma^{2}} [ \beps_{1},\dots,\beps_{n}].
		\end{align}
		Hence repeating the same argument as before, using concentration inequalities for the following random quantities $\mathtt{Q}_{1},\mathtt{Q}_{2}, \mathtt{Q}_{3}$ we defined in \eqref{eq:def_Q}, one can bound the size of $G(\bZ^{\star},\tau)$ with high probability. The rest of the details are omitted. 
	\end{proof}

	\section{Auxiliary computations}
	
	\begin{remark}\label{rmk:thm1_H_parameter}
		Denoting $\xi=\xi'n$ and $\lambda=\lambda'n$, the condition $L/\mu$ in Theorem \ref{thm:SMF_LPGD} for SMF-$\W$  reduces to 
		{\small
			\begin{align}
				&\frac{L^{*}}{\mu^{*}}<3 \,\, \Rightarrow \,\, \left(\frac{L^{*}}{6} <\xi' <\frac{3\mu^{*}}{2}, \quad 0\le  \lambda' <\frac{6\xi'-L^{*}}{2} \right) \cup \left( \xi'>\frac{3\mu^{*}}{2},\quad \frac{2\xi'-3\mu^{*}}{6}< \lambda' <\frac{6\xi'-L^{*}}{2} \right) \\
				&\frac{L^{*}}{\mu^{*}}\ge 3 \,\, \Rightarrow \,\, \left(\frac{L^{*}-\mu^{*}}{4} <\xi' <\frac{3(L^{*}-\mu^{*})}{4}, \,\, \frac{L^{*}-3\mu^{*}}{4}< \lambda' <\frac{6\xi'-L^{*}}{2} \right) \\
				&\hspace{5cm} \cup \left( \xi'>\frac{3(L^{*}-\mu^{*})}{2},\,\, \frac{2\xi'-3\mu^{*}}{6}< \lambda' <\frac{6\xi'-L^{*}}{2} \right).
			\end{align}
		}
	\end{remark}

		\section{Auxiliary lemmas}

		\begin{lemma}\label{lem:gradient_mapping}
			Fix a differentiable function $f:\R^{p}\times \R$ and a convex set $\Param\subseteq \R^{p}$. Fix $\tau>0$ and \begin{align}
				G(\param, \tau):= \frac{1}{\tau}(\param - \Pi_{\Param}(\param - \tau\nabla f(\param))). 
			\end{align} 
			Then for each $\param\in \Param$, $\lVert G(\param,\tau) \rVert\le \lVert \nabla f(\param) \rVert$. 
		\end{lemma}
		
		\begin{proof}
			The assertion is clear if $\lVert G(\param,\tau) \rVert=0$, so we may assume $\lVert G(\param,\tau) \rVert>0$. 
			Denote $\hat{\param}:=\Pi_{\Param}(\param- \tau\nabla f(\param))$. Note that 
			\begin{align}
				\hat{\param}= \argmin_{\param'} \, \lVert \param - \tau\nabla f(\param) - \param' \rVert^{2},
			\end{align}
			so by the first-order optimality condition, 
			\begin{align}
				\langle \hat{\param} - \param + \tau\nabla f(\param) ,\, \param' -\hat{\param}  \rangle \ge 0 \quad \forall \param'\in \Param. 
			\end{align}
			Plugging in $\param'=\param$ and using Cauchy-Schwarz inequality, 
			\begin{align}
				\tau^{2} \lVert G(\param,\tau) \rVert^{2} =  \lVert \param - \hat{\param} \rVert^{2} \le \tau \langle \nabla f(\param),\, \param-\hat{\param} \rangle \le \tau \lVert \nabla f(\param) \rVert \, \tau \lVert G(\param, \tau)  \rVert.
			\end{align}
			Hence the assertion follows by dividing both sides by $\tau^{2}\lVert G(\param,\tau) \rVert>0$. 
		\end{proof}

		\section{Experimental details}
		\label{sec:experimental_details}
		
		All numerical experiments were performed on a 2022 Macbook Air with M1 chip and 16 GB of RAM. 
		
		\subsection{Experiments on semi-synthetic MNIST dataset} 
		We give more details on the semi-synthetic MNIST 
		we used in the experiment in Figure \ref{fig:benchmark_MNIST}. Denote $p=28^{2}=784$, $n=500$, $r=2$, and $\kappa=1$. First, we randomly select 10 images each from digits '2' and '5'. Vectorizing each image as a column in $p=784$ dimension, we obtain a true factor matrix for features $\W_{\textup{true}, X}\in \R^{p\times r}$. Similarly, we randomly sample 10 images of each from digits '4' and '7' and obtain the true factor matrix of labels $\W_{\textup{true}, \Y}\in \R^{p\times r}$. Next, we sample a code matrix $\H_{\text{true}} \in \R^{r \times n}$ whose entries are i.i.d. with the uniform distribution $U([0,1])$. Then the `pre-feature' matrix $\X_0 \in \R^{p\times n}$ of vectorized synthetic images is generated by $\W_{\textup{true}, X} \H_\text{true}$. The feature matrix $\X_{\textup{data}}\in \R^{p\times n}$ is then generated by adding an independent Gaussian noise $\eps_{j}\sim N(\mathbf{0}, \sigma^{2} I_{p})$ to the $j$th column of $\X_{0}$ for $j=1,\dots,n,$ with $\sigma=0.5$. We generate the binary label matrix $\Y=[y_{1},\dots,y_{n}]  \in \{0,1\}^{1\times n }$   (recall $\kappa=1$) as follows: Each entry $y_{i}$ is an independent Bernoulli variable with probability $p_{i} = \left( 1+\exp{(-\Beta_{\textup{true},\Y}^{T} \W_{\textup{true},\Y}^{T} \X_{\textup{data}}[:,i]  )}  \right)^{-1}$, where 
		$\Beta_{\textup{true}, \Y} = [1,-1]$. No auxiliary features were used for the semi-synthetic dataset (i.e., $q=0$).

		\subsection{Experiments on the Job postings dataset} 
		
		Next, we give details on the job postings dataset \cite{fakejob_data}. There are 17,880 postings and 15 variables in the dataset including binary variables, categorical variables, and textual information of \textit{job description}. Among the 17,880 postings, 17,014 are true job postings (95.1\%) and 866 are fraudulent postings (4.84\%). This reveals a significant class imbalance, where the number of true postings greatly outweighs fraudulent ones, making this class imbalance a noteworthy characteristic of the dataset. In our analysis, we have coded the fake job postings as positive examples and the true job postings as negative examples.
		
		In our experiments, we represented each job posting as a $p=2480$ dimensional word frequency vector computed from its \textit{job description} and augmented with $q=72$ auxiliary features of binary and categorical variables, including indicators that specify whether a job posting has a company logo or if the posted job is in the United States. For computing the word frequency vectors, we represent the job description variable as a term/document frequency matrix with Term Frequency-Inverse Document Frequency (TF-IDF) normalization \cite{pedregosa2011scikit}. Common words that appear in all documents are assigned lower importance, while words specific to particular documents are deemed more significant. In our analysis, we utilized the 2,480 most frequent words for further examination.
		
		\subsection{Details on CNN and FFNN}

		For the task of classifying microarray data into cancer classes, we compared the performance of our method with both CNN and FFNN in Figure \ref{fig:pancreatic_cancer}. Specifically, the CNN architecture was designed with a convolutional layer with 32 filters and a kernel size of 3, followed by an average pooling layer with a pool size of 2. Subsequently, a second convolutional layer with 64 filters and a kernel size of 3 was integrated, further followed by another average pooling layer with the same pool size. The architecture was finalized with a flatten layer, a fully connected layer of 128 neurons activated by ReLU, a dropout layer with a rate of 0.5, and a final fully connected layer with a sigmoid activation. On the other hand, the FFNN consisted of a fully connected layer featuring 64 neurons with ReLU, followed by a dropout layer with a regularization rate of 0.5. A subsequent fully connected layer with 32 neurons activated by the ReLU was incorporated, followed by a fully connected layer with a sigmoid function. This comparative analysis was repeated five times, consistent with the procedure outlined in the main paper. 
		
		An intriguing observation emerges from our benchmarking analysis. While the FFNN's performance on the breast cancer dataset is comparable to ours (The LPGD algorithms for SMF), the overall performance of CNN is notably inferior to ours. This disparity can primarily be attributed to the small sample size of the training set (145 samples for breast cancer and 25 samples for pancreatic cancer) in comparison to the substantial dimensionality of gene features (exceeding 30,000 features). We note that obtaining a substantial volume of biomedical data for cancer research is very expensive, making it challenging to feasibly train complex models such as deep neural networks. The significance of our approach becomes evident in its ability to retain robust performance even when facing the challenges posed by a restricted sample size and a complex high-dimensional feature landscape. Moreover, our method augments this resilience with the advantage of interpretability.


		
		
		
		

	\end{document}